\documentclass{article}
\usepackage[final]{neurips_2019}

\usepackage{hyperref}       % hyperlinks
\usepackage{microtype} 
\usepackage{algorithm}
\usepackage{algpseudocode}

\usepackage{mathtools}
\usepackage{amsmath}
\usepackage{amsthm}
\usepackage{bbm}
\usepackage{amsfonts}
\usepackage{amssymb}
\usepackage{wrapfig}
\usepackage{graphicx}
\usepackage[mathscr]{euscript}

\usepackage{MnSymbol}
\DeclareMathAlphabet\mathbb{U}{msb}{m}{n}
\usepackage{xpatch}

\def\Rset{\mathbb{R}}

\DeclareMathOperator*{\E}{\mathbb E}

\DeclareMathOperator*{\sgn}{sgn}

\DeclareMathOperator{\Var}{Var}
\let\Pr\relax 
\DeclareMathOperator*{\Pr}{\mathbb{P}}
\let\P\relax 
\DeclareMathOperator*{\P}{\mathbb{P}}

\newtheorem{theorem}{Theorem}
\newtheorem{definition}{Definition}
\newtheorem{lemma}{Lemma}

\newcommand{\cA}{\mathcal{A}}

\newcommand{\cG}{\mathcal{G}}
\newcommand{\cH}{\mathcal{H}}

\newcommand{\cW}{\mathcal{W}}
\newcommand{\cX}{\mathcal{X}}

\newcommand{\cZ}{\mathcal{Z}}

\newcommand{\sD}{{\mathscr D}}
\newcommand{\sF}{{\mathscr F}}
\newcommand{\sG}{{\mathscr G}}
\newcommand{\sH}{{\mathscr H}}

\newcommand{\sL}{{\mathscr L}}

\newcommand{\sU}{{\mathscr U}}
\newcommand{\sX}{{\mathscr X}}
\newcommand{\sY}{{\mathscr Y}}
\newcommand{\sZ}{{\mathscr Z}}

\newcommand{\supcvstab}{\chi}
\newcommand{\avgcvstab}{\bar{\chi}}
\newcommand{\supdiam}{\Delta}
\newcommand{\maxdiam}{\Delta_{\max}}
\newcommand{\avgdiam}{\bar{\Delta}}

\newcommand{\R}{\mathfrak R}
\newcommand{\sfB}{\mathsf B}
\newcommand{\sfD}{\mathsf D}
\newcommand{\sfS}{\mathsf S}

\newcommand{\bsigma}{{\boldsymbol \sigma}}

\newcommand{\h}{\widehat}
\newcommand{\ov}{\overline}

\newcommand{\e}{\epsilon}
\newcommand{\set}[2][]{#1 \{ #2 #1 \} }
\newcommand{\ignore}[1]{}

\hypersetup{
  colorlinks   = true,
  urlcolor     = blue,
  linkcolor    = blue,
  citecolor   = blue
}

\usepackage{color-edits}
\addauthor{mm}{blue}
\addauthor{sk}{red}

\title{Hypothesis Set Stability and Generalization}

\author{Dylan J. Foster\\
Massachusetts Institute of Technology\\
\texttt{dylanf@mit.edu}
\And
Spencer Greenberg\\
Spark Wave \\
\texttt{admin@sparkwave.tech}
\And
\phantom{XXXX}Satyen Kale\\
\phantom{XXXX}Google Research \\
\phantom{XXXX}\texttt{satyen@satyenkale.com}
\And
\phantom{XXX}Haipeng Luo\\
\phantom{XX}University of Southern California \\
\phantom{XXX}\texttt{haipengl@usc.edu}
\And
\phantom{Xi}Mehryar Mohri\\
\phantom{Xi}Google Research and Courant Institute \\
\phantom{Xi}\texttt{mohri@google.com}
\And
Karthik Sridharan \\
Cornell University \\
\texttt{sridharan@cs.cornell.edu}
}
\begin{document}

\maketitle

\begin{abstract}
  We present a study of generalization for data-dependent hypothesis
  sets. We give a general learning guarantee for data-dependent
  hypothesis sets based on a notion of transductive Rademacher
  complexity.  Our main result is a generalization bound for
  data-dependent hypothesis sets expressed in terms of a notion of
  \emph{hypothesis set stability} and a notion of Rademacher
  complexity for data-dependent hypothesis sets that we
  introduce. This bound admits as special cases both standard
  Rademacher complexity bounds and algorithm-dependent uniform
  stability bounds. We also illustrate the use of these learning
  bounds in the analysis of several scenarios.
\end{abstract}

\section{Introduction}

Most generalization bounds in learning theory hold for a fixed
hypothesis set, selected before receiving a sample. This includes
learning bounds based on covering numbers, VC-dimension,
pseudo-dimension, Rademacher complexity, local Rademacher complexity,
and other complexity measures \citep{Pollard1984,Zhang2002,Vapnik1998,
  KoltchinskiiPanchenko2002,BartlettBousquetMendelson2002}.
Some alternative guarantees have also been derived for specific
algorithms. Among them, the most general family is that of uniform
stability bounds given by \cite{BousquetElisseeff2002}. These bounds
were recently significantly improved by \citet{FeldmanVondrak2019},
who proved guarantees that are informative, even when the stability
parameter $\beta$ is only in $o(1)$, as opposed to $o(1/\sqrt{m})$.
The $\log^2 m$ factor in these bounds was later reduced to
$\log m$ by
\citet{BousquetKlochkovKlochkov2019}.
Bounds for a restricted class of algorithms were also recently
presented by \citet{Maurer2017}, under a number of assumptions on the
smoothness of the loss function. Appendix~\ref{app:stability} gives
more background on stability.

In practice, machine learning engineers commonly resort to hypothesis
sets depending on the \emph{same sample} as the one used for
training. This includes instances where a regularization, a feature
transformation, or a data normalization is selected using the training
sample, or other instances where the family of predictors is
restricted to a smaller class based on the sample received. In other
instances, as is common in deep learning, the data representation and
the predictor are learned using the same sample. In ensemble learning,
the sample used to train models sometimes coincides with the one used
to determine their aggregation weights.  However, standard
generalization bounds are not directly applicable for these
scenarios since they assume a fixed hypothesis set.

\subsection{Contributions of this paper.}

\begin{figure}[t]
\centering
\includegraphics[scale=.4]{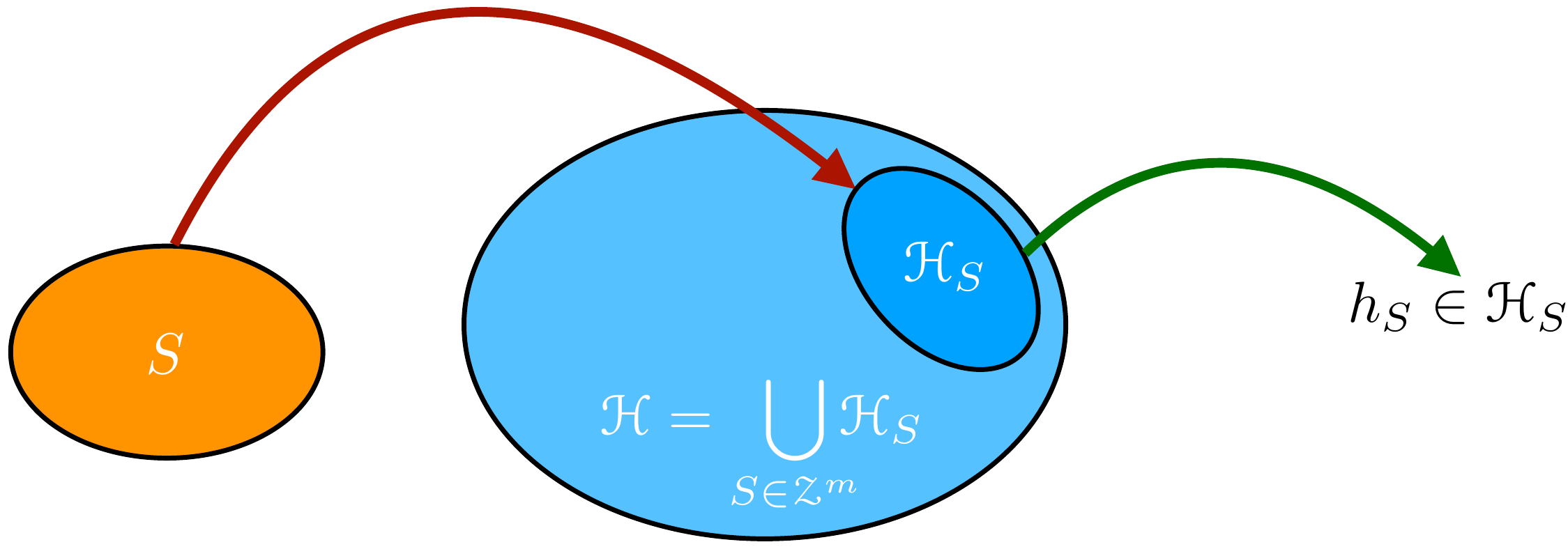}
\caption{Decomposition of the learning algorithm's hypothesis
  selection into two stages. In the first stage, the algorithm
  determines a hypothesis $\sH_S$ associated to the training sample
  $S$ which may be a small subset of the set of all hypotheses that
  could be considered, say $\sH = \bigcup_{S \in \sZ^m} \sH_S$. The
  second stage then consists of selecting a hypothesis $h_S$ out of
  $\sH_S$.}
\vskip -.1in
\label{fig:hss}
\end{figure}

{\bf 1. Foundational definitions of data-dependent hypothesis sets.} 
% This paper studies generalization in a broad setting that admits as
% special cases both that of standard learning bounds for fixed
% hypothesis sets based on some complexity measure, and that of
% algorithm-dependent uniform stability bounds.  
We present foundational definitions of learning algorithms that rely on 
% an extensive study of generalization for 
\emph{data-dependent} hypothesis sets. Here, the algorithm decomposes into two stages: 
% with a hypothesis set $\sH_S$ selected after receiving the training
% sample $S$. This defines two stages for the learning algorithm:
a first stage where the algorithm, on receiving the sample $S$,
chooses a hypothesis set $\sH_S$ dependent on $S$, and a second stage
where a hypothesis $h_S$ is selected from $\sH_S$. Standard
generalization bounds correspond to the case where $\sH_S$ is equal to
some fixed $\sH$ independent of $S$. Algorithm-dependent analyses,
such as uniform stability bounds, coincide with the case where $\sH_S$
is chosen to be a singleton $\sH_S = \set{h_S}$. Thus, the scenario we
study covers both existing settings and other intermediate
scenarios. Figure~\ref{fig:hss} illustrates our general scenario.

{\bf 2. Learning bounds via transductive
Rademacher complexity.} 
% We present a series of results for generalization with data-dependent
% hypothesis sets. 
We present general learning bounds for data-dependent hypothesis sets using a notion of transductive Rademacher complexity (Section~\ref{sec:hss2}). These bounds hold for arbitrary bounded losses and improve upon previous guarantees given by \cite{Gat2001} and \cite{CannonEttingerHushScovel2002} for the binary
loss, which were expressed in terms of a notion of shattering
coefficient adapted to the data-dependent case, and are more explicit
than the guarantees presented by \cite{Philips2005}[corollary~4.6 or theorem~4.7].  Nevertheless, such bounds may often not be sufficiently informative, since they ignore the relationship between hypothesis sets based on similar samples.

{\bf 3. Learning bounds via hypothesis set stability.}
We provide finer generalization bounds based on the key notion of
\emph{hypothesis set stability} that we introduce in this paper. This notion admits algorithmic stability as a special case, when the hypotheses sets are reduced to singletons. We also introduce a new notion of \emph{Rademacher complexity for data-dependent hypothesis sets}. Our main results are generalization bounds (Section~\ref{sec:hss}) for stable data-dependent hypothesis sets expressed in terms of the hypothesis set stability parameter, our notion of Rademacher complexity, and a notion of \emph{cross-validation stability} that, in turn, can be upper-bounded by the diameter of the family of hypothesis sets. 
% Our first learning bound
% (Section~\ref{sec:hss}) is expressed in terms of a finer notion of
% diameter but admits a dependency in terms of the stability parameter
% $\beta$ similar to that of uniform stability bounds of
% \citet{BousquetElisseeff2002}. In Section~\ref{sec:hss-diff}, we use
% proof techniques from the differential privacy literature
% \citep{SteinkeUllman2017,BassilyNissimSmithSteinkeStemmerUllman2016,FeldmanVondrak2018}
% to derive a learning bound expressed in terms of a somewhat coarser
% definition of diameter but with a more favorable dependency on
% $\beta$, matching the dependency of the recent more favorable bounds
% of \citet{FeldmanVondrak2018}. 
Our learning bounds admit as special cases both standard Rademacher complexity bounds and algorithm-dependent uniform stability bounds. 

{\bf 4. New generalization bounds for specific learning applications.}
In section~\ref{sec:applications} (see also
Appendix~\ref{app:other-applications}), we illustrate the generality
and the benefits of our hypothesis set stability learning bounds by
using them to derive new generalization bounds for several learning
applications. To our knowledge, there is no straightforward analysis based on previously existing tools that yield these generalization bounds. These
applications include: (a) \emph{bagging} algorithms that may employ
non-uniform, data-dependent, averaging of the base predictors, (b)
\emph{stochastic strongly-convex optimization} algorithms based on an
average of other stochastic optimization algorithms, (c) \emph{stable
  representation} learning algorithms, which first learn a data
representation using the sample and then learn a predictor on top of
the learned representation, and (d) \emph{distillation} algorithms,
which first compute a complex predictor using the sample and then use
it to learn a simpler predictor that is close to it.

% In
% Appendix~\ref{app:extensions}, we briefly discuss several extensions
% of our framework and results, including the extension to almost
% everywhere hypothesis set stability as in \citep{KutinNiyogi2002}.
% The next section introduces the definitions and properties used in our
% analysis.

\subsection{Related work on data-dependent hypothesis sets.} 

\cite{ShaweTaylorBartlettWilliamsonAnthony1998} presented an analysis
of structural risk minimization over data-dependent hierarchies based
on a concept of \emph{luckiness}, which generalizes the notion of
margin of linear classifiers.  Their analysis can be viewed as an
alternative and general study of data-dependent hypothesis sets, using
luckiness functions and $\omega$-smallness (or $\omega$-smoothness)
conditions. A luckiness function helps decompose a hypothesis set into
\emph{lucky sets}, that is sets of functions \emph{luckier} than a
given function.
% The $\omega$-smallness condition requires that the
% size of the family of loss functions corresponding to the lucky set of
% any function $f$ with respect to a double-sample, measured by packing
% or covering numbers, be bounded with high probability by a function
% $\omega$ of the luckiness of $f$ on the sample.
%
The luckiness framework is attractive and the notion of luckiness, for
example margin, can in fact be combined with our results. However,
finding pairs of truly data-dependent luckiness and $\omega$-smallness
functions, other than those based on the margin and the empirical
VC-dimension, is quite difficult, in particular because of the very
technical $\omega$-smallness condition
\citep[see][p.~70]{Philips2005}. In contrast, hypothesis set stability
is simpler and often easier to bound.
The notions of luckiness and $\omega$-smallness have also been used by
\cite{HerbrichWilliamson2002} to derive algorithm-specific
guarantees. In fact, the authors show a connection with algorithmic
stability (not hypothesis set stability), at the price of a guarantee
requiring the strong condition that the stability parameter be in
$o(1/m)$, where $m$ is the sample size
\citep[see][pp.~189-190]{HerbrichWilliamson2002}.

Data-dependent hypothesis classes are conceptually related to the
notion of data-dependent priors in PAC-Bayesian generalization
bounds. \citet{Catoni2007} developed localized PAC-Bayes analysis by
using a prior defined in terms of the data generating
distribution. This work was extended by
\citet{LeverLavioletteShaweTaylor2013} who proved sharp risk bounds
for stochastic exponential weights
algorithms. \citet{ParradoHernandezAmbroladzeShaweTaylorSun2012}
investigated the possibility of learning the prior from a separate
data set, as well as priors obtained via computing a data-dependent
bound on the KL term. More closely related to this paper is the work
of \citet{DziugaiteRoy2018a,DziugaiteRoy2018b}, who develop PAC-Bayes
bounds by choosing the prior via a data-dependent differentially
private mechanism, and also showed that weaker notions than
differential privacy also suffice to yield valid bounds.
In Appendix~\ref{sec:PAC-Bayes}, we give a more detailed discussion of
PAC-Bayes bounds, in particular to show how finer PAC-Bayes bounds
than standard ones can be derived from Rademacher complexity bounds,
here with an alternative analysis and constants than
\citep{KakadeSridharantTewari2008} and how data-dependent PAC-Bayes
bounds can be derived from our data-dependent Rademacher complexity
bounds. More discussion on data-dependent priors can be found in
Appendix~\ref{sec:data-dependent-priors}. 

\ignore{
This work may be viewed as
complementary to our own; in particular, PAC-Bayes bounds
generally\footnote{``Derandomized'' PAC-Bayes bounds are known for
  certain cases where a single hypothesis is output, such as SVM.}
apply to randomized mixtures of hypotheses, whereas our bounds apply
to any algorithm that outputs a single hypothesis.  More discussion on
data-dependent priors can be found in
Appendix~\ref{sec:data-dependent-priors}.
}

\section{Definitions and Properties}
\label{sec:definitions}

Let $\sX$ be the input space and $\sY$ the output space, and define $\sZ := \sX \times \sY$  We denote by
$\sD$ the unknown distribution over $\sX \times \sY$ according to which samples are drawn.

The hypotheses $h$ we consider map $\sX$ to a set $\sY'$ sometimes
different from $\sY$. For example, in binary classification, we may
have $\sY = \set{-1, +1}$ and $\sY' = \Rset$.  Thus, we denote by
$\ell\colon \sY' \times \sY \to [0, 1]$ a loss function defined on
$\sY' \times \sY$ and taking non-negative real values bounded by one.
We denote the loss of a hypothesis $h\colon \sX \to \sY'$ at point
$z = (x, y) \in \sX \times \sY$ by $L(h, z) = \ell(h(x), y)$. 
We denote by $R(h)$ the generalization error
or expected loss of a hypothesis $h \in \sH$ and 
by $\h R_S(h)$ its empirical loss over a sample
$S = (z_1, \ldots, z_m)$:
\[
R(h) = \E_{z \sim \sD} [L(h, z)] \qquad \h R_S(h) = \E_{z \sim S}
[L(h, z)] = \frac{1}{m} \sum_{i = 1}^m L(h, z_i).
\]
In the general framework we consider, a hypothesis set depends on the
sample received. We will denote by $\sH_S$ the hypothesis set
depending on the labeled sample $S \in \sZ^m$ of size
$m \geq 1$. We assume that $\sH_S$ is invariant to the ordering of the points in $S$.

\begin{definition}[Hypothesis set uniform stability]
  Fix $m \geq 1$. We will say that a family of data-dependent
  hypothesis sets $\cH = (\sH_S)_{S \in \sZ^m}$ is \emph{$\beta$-uniformly
    stable} (or simply $\beta$-stable) for some $\beta \geq 0$, if for any two
  samples $S$ and $S'$ of size $m$ differing only by one point, the
  following holds:
\begin{equation}
\label{eq:hss}
  \forall h \in \sH_S, \exists h' \in \sH_{S'} \colon \,
\forall z \in \sZ, | L(h, z) - L(h', z) | \leq \beta.
\end{equation}
\end{definition}
Thus, two hypothesis sets derived from samples differing by one
element are close in the sense that any hypothesis in one admits a
counterpart in the other set with $\beta$-similar losses. A closely related notion is the \emph{sensitivity} of a function $f\colon \sZ^m \to \Rset$. Such a function $f$ is called $\beta$-sensitive if for any two
  samples $S$ and $S'$ of size $m$ differing only by one point, we have $|f(S) - f(S')| \leq \beta$.

We also introduce a new notion of Rademacher complexity for
data-dependent hypothesis sets.  To introduce its definition, for any
two samples $S, T \in \sZ^m$ and a vector of Rademacher variables
$\bsigma$, denote by $S_{T, \bsigma}$ the sample derived from $S$ by
replacing its $i$th element with the $i$th element of $T$, for all
$i \in [m] = \set{1, 2, \ldots, m}$ with $\sigma_i = -1$. We will use $\sH_{S, T}^\bsigma$ to
denote the hypothesis set $\sH_{S_{T, \bsigma}}$.

\begin{definition}[Rademacher complexity of data-dependent hypothesis sets]
  Fix $m \geq 1$. The \emph{empirical Rademacher complexity $\h
    \R^\diamond_{S, T}(\cH)$ 
and the Rademacher complexity $\R^{\diamond}_{m}(\cH)$ of a
    family of data-dependent hypothesis sets
    $\cH = (\sH_S)_{S \in \sZ^m}$} for two samples
  $S = (z^S_1, \ldots, z^S_m)$ and $T = (z^T_1, \ldots, z^T_m)$ in
  $\sZ^m$ are defined by
\begin{equation}
\label{eq:deRad}
\h \R^\diamond_{S, T}(\cH) = 
\frac{1}{m} \E_{\bsigma} \Bigg[\sup_{h \in \sH_{S, T}^\bsigma} \sum_{i =
  1}^m \sigma_i h(z^T_i) \Bigg] \qquad
\R^\diamond_{m}(\cH) = 
\frac{1}{m} \E_{\substack{S, T \sim \sD^m\\\bsigma}} \Bigg[\sup_{h \in \sH_{S, T}^\bsigma} \sum_{i =
  1}^m \sigma_i h(z^T_i) \Bigg].
\end{equation}
\end{definition}
When the family of data-dependent hypothesis sets $\cH$ is
$\beta$-stable with $\beta = O(1/m)$, the empirical Rademacher
complexity $\h \R^\diamond_{S, T}(\cG)$ is sharply concentrated around
its expectation $\R^\diamond_{m}(\cG)$, as with the standard empirical
Rademacher complexity (see Lemma~\ref{lemma:concentration}).

Let $\sH_{S, T}$ denote the union of all hypothesis sets based on
$\sU = \set{U \colon U \subseteq (S \cup T), U \in \sZ^m}$, the
subsamples of $S \cup T$ of size $m$:
$\sH_{S, T} = \bigcup_{U \in \sU} \sH_U$. Since for any $\bsigma$, we have
$\sH_{S, T}^\bsigma \subseteq \sH_{S, T}$, the following simpler upper
bound in terms of the standard empirical Rademacher complexity of
$\sH_{S, T}$ can be used for our notion of empirical Rademacher
complexity:
\begin{align*}
\R^\diamond_{m}(\cH) 
& \leq \frac{1}{m} 
\E_{\substack{S, T \sim \sD^m\\\bsigma}} \bigg[\sup_{h \in \sH_{S, T}} \sum_{i =
  1}^m \sigma_i h(z^T_i) \bigg] = \E_{\substack{S, T \sim \sD^m}} \Big[ \h \R_T(\sH_{S, T}) \Big],
\end{align*}
where $\h \R_T(\sH_{S, T})$ is the standard empirical Rademacher\footnote{Note that the standard definition of Rademacher complexity assumes that hypothesis sets are not data-dependent, however the definition remains valid for data-dependent hypothesis sets.}
complexity of $\sH_{S, T}$ for the sample $T$. Some properties of our notion of Rademacher complexity are given in Appendix~\ref{app:rademacher}.

\ignore{
The Rademacher complexity of data-dependent hypothesis sets can be
bounded by $\E_{\substack{S, T \sim \sD^m}} \Big[ \h \R_S(\sH_{S, T}) \Big]$, as indicated previously. It can also be bounded directly, as illustrated
by the following example of data-dependent hypothesis sets of linear
predictors.  For any sample $S = (x^S_1, \ldots, x^S_m) \in \Rset^N$,
define the hypothesis set $\sH_S$ as follows:
\[
\sH_S = \set[\bigg]{x \mapsto w^S \cdot x \colon \ w^S = \sum_{i = 1}^m
  \alpha_i x^S_i, \| \alpha \|_1 \leq \Lambda_1 },
\]
where $\Lambda_1 \geq 0$.
Define $r_T$ and $r_{S \cup T}$ as follows:
$r_T = \sqrt{\frac{\sum_{i = 1}^m \| x^T_i \|_2^2}{m}}$ and  $r_{S \cup T} =
\max_{x \in S \cup T} \| x \|_2$.
Then, it can be shown that 
the empirical Rademacher complexity of the family of
data-dependent hypothesis sets $\cH = (\sH_S)_{S \in \sX^m}$ can be
upper-bounded as follows (Lemma~\ref{lemma:rad}):
\[
\h \R^\diamond_{S, T}(\cH) 
\leq r_T \, r_{S \cup T} \Lambda_1 \sqrt{\frac{2 \log (4m)}{m}} 
\leq r^2_{S \cup T} \Lambda_1 \sqrt{\frac{2 \log (4m)}{m}}.
\]
Notice that the bound on the Rademacher complexity is non-trivial
since it depends on the samples $S$ and $T$, while a standard
Rademacher complexity for non-data-dependent hypothesis set
containing $\sH_S$ would require taking a maximum over all
samples $S$ of size $m$.
Other upper bounds are given in Appendix~\ref{app:rademacher}.} % ends ignore

Let $\sG_S$ denote the family of loss functions associated to $\sH_S$:
\begin{equation} \label{eq:G-def}
\sG_S = \set{z \mapsto L(h, z) \colon h \in \sH_S},  
\end{equation}
and let $\cG = (\sG_S)_{S \in \sZ^m}$ denote the family of hypothesis sets
$\sG_S$. Our main results will be expressed in terms of
$\R^\diamond_m(\cG)$. When the loss function $\ell$ is
$\mu$-Lipschitz, by Talagrand's contraction lemma
\citep{LedouxTalagrand1991}, in all our results, $\R^\diamond_m(\cG)$
can be replaced by $\mu \E_{\substack{S, T \sim \sD^m}} [ \h \R_T(\sH_{S, T})]$.

Rademacher complexity is one way to measure the capacity of the family
of data-dependent hypothesis sets. We also derive learning bounds in
situations where a notion of \emph{diameter} of the hypothesis sets is
small. We now define a notion of \emph{cross-validation stability} and
diameter for data-dependent hypothesis sets. In the following, for a
sample $S$, $S^{z \leftrightarrow z'}$ denotes the sample obtained
from $S$ by replacing $z \in S$ by $z' \in \sZ$. 
% The notion measures the maximal change in loss of a hypothesis on a
% training example and the loss of a hypothesis on the same training
% example, when the hypothesis is chosen from the hypothesis set
% corresponding to the a sample where the training example in question
% is replaced by a newly sampled example.

\begin{definition}[Hypothesis set Cross-Validation (CV) stability, diameter]
  Fix $m \geq 1$. For some $\avgcvstab, \supcvstab, \avgdiam, \supdiam, \maxdiam \geq 0$, we
  say that a family of data-dependent hypothesis sets
  $\cH = (\sH_S)_{S \in \sZ^m}$ has \emph{average CV-stability} $\avgcvstab$, \emph{CV-stability} $\supcvstab$, \emph{average diameter} $\avgdiam $, \emph{diameter} $\supdiam$ and \emph{max-diameter} $\maxdiam$ if the
  following hold:
\begin{align}
\E_{S \sim \sD^m} \E_{z' \sim\sD, z \sim S}\bigg[\sup_{h \in \sH_S, h' \in \sH_{S^{z \leftrightarrow z'}}} L(h', z) - L(h, z) \bigg] & \leq \avgcvstab  
\\
\sup_{S \in \sZ^m} \E_{z' \sim\sD, z \sim S}\bigg[\sup_{h \in \sH_S, h' \in \sH_{S^{z \leftrightarrow z'}}} L(h', z) - L(h, z) \bigg] & \leq \supcvstab \label{eq:CVStab}\\
\E_{S \sim \sD^m} \E_{z \sim S} \bigg[\sup_{h, h' \in
  \sH_S} L(h', z) - L(h, z) \bigg] &\leq \avgdiam \label{eq:avgdiam} \\
\sup_{S \in \sZ^m} \E_{z \sim S} \bigg[\sup_{h, h' \in
  \sH_S} L(h', z) - L(h, z) \bigg] &\leq \supdiam \label{eq:diam} \\
\sup_{S \in \sZ^m} \max_{z \in S} \bigg[\sup_{h, h' \in
  \sH_S} L(h', z) - L(h, z) \bigg] &\leq \maxdiam. \label{eq:maxdiam}
\end{align}
\end{definition}

CV-stability of hypothesis sets can be bounded in terms of their
stability and diameter (see straightforward proof in
Appendix~\ref{app:cvstab-diam-bound}).

\begin{lemma} \label{lem:cvstab-diam-bound}
  Suppose a family of data-dependent hypothesis sets $\cH$ is
  $\beta$-uniformly stable. Then if it has diameter $\supdiam$, then it is $(\supdiam + \beta)$-CV-stable, and if it has average diameter $\avgdiam$ then it is $(\avgdiam + \beta)$-average CV-stable.
\end{lemma}

\section{General learning bound for data-dependent hypothesis sets}
\label{sec:hss2}

In this section, we present general learning bounds for data-dependent
hypothesis sets that do not make use of the notion of hypothesis set
stability. One straightforward idea to derive such guarantees for data-dependent hypothesis sets is to replace the hypothesis set $\sH_S$ depending on the observed sample $S$ by the union of all such hypothesis sets over all samples of size $m$, $\ov \sH_m = \bigcup_{S \in \sZ^m}
\sH_S$. However, in general, $\ov \sH_m$ can be very rich, which can
lead to uninformative learning bounds. A somewhat better alternative
consists of considering the union of all such hypothesis sets for
samples of size $m$ included in some supersample $U$ of size $m + n$,
with $n \geq 1$,
$\ov \sH_{U, m} = \bigcup_{\substack{S \in \sZ^m\\ S \subseteq U
    \mspace{14mu}}} \sH_S$. We will derive learning guarantees based
on the maximum \emph{transductive Rademacher complexity} of
$\ov \sH_{U, m}$. There is a trade-off in the choice of $n$: smaller
values lead to less complex sets $\ov \sH_{U, m}$, but they also lead
to weaker dependencies on sample sizes. Our bounds are more refined
guarantees than the shattering-coefficient bounds originally given for
this problem by \cite{Gat2001} in the case $n = m$, and later by
\cite{CannonEttingerHushScovel2002} for any $n \geq 1$. They also
apply to arbitrary bounded loss functions and not just the binary
loss.
They are expressed in terms of the following notion of
\emph{transductive Rademacher complexity for data-dependent
hypothesis sets}:
\[
\h \R^\diamond_{U, m}(\cG) = \E_\bsigma \Bigg[ \sup_{h \in \ov \sH_{U, m}} \frac{1}{m + n}
\sum_{i = 1}^{m + n} \sigma_i L(h, z^U_i) \Bigg],
\]
where $U = (z^U_1, \ldots, z^U_{m + n}) \in \sZ^{m + n}$ and where
$\bsigma$ is a vector of $(m + n)$ independent random variables taking
value $\frac{m + n}{n}$ with probability $\frac{n}{m + n}$, and
$-\frac{m + n}{m}$ with probability $\frac{m}{m + n}$.  Our notion of
transductive Rademacher complexity is simpler than that of
\cite{ElYanivPechyony2007} (in the data-independent case) and leads to
simpler proofs and guarantees.  A by-product of our analysis is
learning guarantees for standard transductive learning in terms of
this notion of transductive Rademacher complexity, which can be of
independent interest.

\begin{theorem}
\label{th:hss2}
Let $\cH = (\sH_S)_{S \in \sZ^m}$ be a family of data-dependent
hypothesis sets. Let $\cG$ be defined as in \eqref{eq:G-def}. Then,
for any $\delta > 0$, with probability at least $1-\delta$ over the
choice of the draw of the sample $S \sim \sZ^m$, the following
inequality holds for all $h \in \sH_S$:
\begin{equation*}
R(h) \leq \h R_S(h) + \max_{U \in \sZ^{m + n}} 2\h \R^\diamond_{U, m}(\cG) + 3\sqrt{\left(\tfrac{1}{m} + \tfrac{1}{n}\right) \log(\tfrac{2}{\delta})} + 2\sqrt{\left(\tfrac{1}{m} + \tfrac{1}{n}\right)^3 mn}.
\end{equation*}
% where $\eta = \frac{m + n}{m + n - \frac{1}{2}} \frac{1}{1 - \frac{1}{2 \max\set{m, n}}} \approx 1$. 
% For $m = n$, the inequality becomes:
% \begin{equation*}
% \P\bigg[\sup_{h \in \sH_S} R(h) - \h R_S(h) > \e \bigg]
% \leq \exp\left[ - \frac{m}{\eta} \left[\frac{\e}{2} -
%   \max_{U \in \sZ^{m + n}} \h \R^\diamond_{U, m}(\cG) - 2\sqrt{\frac{\log(2e)}{m}} \right]^2 \right].
% \end{equation*}
\end{theorem}
\begin{proof} (Sketch; full proof in Appendix~\ref{app:hss2}.)
  We use the following symmetrization result, which holds for any
  $\e > 0$ with $m \e^2 \geq 2$ for data-dependent hypothesis sets
  (Lemma~\ref{lemma:symmetrization},
  Appendix~\ref{app:hss2}):
\[
 \P_{S \sim \sD^m} \bigg[ \sup_{h \in \sH_S} R(h) - \h R_S(h) > \e \bigg] 
\leq 2 \P_{\substack{S \sim \sD^m \\ T \sim \sD^n \mspace{6mu}}}
\bigg[ \sup_{h \in \sH_S} \h R_T(h) - \h R_S(h) > \frac{\e}{2}  \bigg] .
\]
To bound the right-hand side, we use an extension of McDiarmid's
inequality to sampling without replacement
\citep{CortesMohriPechyonyRastogi2008} applied to
$\Phi(S) = \sup_{h \in \ov \sH_{U, m}} \h R_T(h) - \h R_S(h)$.
Lemma~\ref{lemma:trans} (Appendix~\ref{app:hss2}) is then used to
bound $\E[\Phi(S)]$ in terms of our notion of transductive Rademacher
complexity.
\end{proof}

\section{Learning bound for stable data-dependent hypothesis sets}
\label{sec:hss}

In this section, we present our main generalization bounds for 
data-dependent hypothesis sets.
% using the notion of Rademacher complexity defined in the previous
% section.  , as well as that of hypothesis set stability.

\begin{theorem}
\label{th:hss}
Let $\cH = (\sH_S)_{S \in \sZ^m}$ be a $\beta$-stable family of
data-dependent hypothesis sets, with $\avgcvstab$ average CV-stability, $\supcvstab$ CV-stability and
$\maxdiam$ max-diameter. Let $\cG$ be defined as in
\eqref{eq:G-def}. Then, for any $\delta > 0$, with probability at
least $1 - \delta$ over the draw of a sample $S \sim \sZ^m$, the
following inequality holds for all $h \in \sH_S$:
\begin{align}
\forall h \in \sH_S, 
R(h) \leq \h R_S(h) + \min\Bigg\{ & \min\left\{2 \R^\diamond_m(\cG), \avgcvstab\right\} + (1 + 2 \beta m) \sqrt{\tfrac{1}{2m}\log(\tfrac{1}{\delta})}, \label{eq:rad-hss} \\
& \sqrt{e} \supcvstab + 4 \sqrt{(\tfrac{1}{m} + 2 \beta) \log(\tfrac{6}{\delta})}, \label{eq:dp-hss}\\
& 48(3\beta + \maxdiam)\log(m)\log(\tfrac{5m^3}{\delta}) + \sqrt{\tfrac{4}{m}\log(\tfrac{4}{\delta})}\Bigg\}. \label{eq:induction-hss}
\end{align}
\end{theorem}

The proof of the theorem is given in Appendix~\ref{app:hss}. The main
idea is to control the sensitivity of the function $\Psi(S, S')$
defined for any two samples $S, S'$, as follows:
\begin{equation*}
\Psi(S, S') = \sup_{h \in \sH_{S}} R(h) - \h R_{S'}(h) .
\end{equation*}
To prove bound~\eqref{eq:rad-hss}, we apply McDiarmid's inequality to
$\Psi(S, S)$, using the $(\frac{1}{m} + 2 \beta)$-sensitivity of
$\Psi(S, S)$, and then upper bound the expectation
$\E_{S \sim \sD^m}[\Psi(S, S)]$ in terms of our notion of Rademacher
complexity. The bound~\eqref{eq:dp-hss} is obtained via a
differential-privacy-based technique, as in \cite{FeldmanVondrak2018},
and \eqref{eq:induction-hss} is a direct consequence of the bound of
\citet{FeldmanVondrak2019} using the observation that an algorithm
that chooses an arbitrary $h \in \sH_S$ is
$O(\beta + \maxdiam)$-uniformly stable in the classical \citep{BousquetElisseeff2002} sense.

Bound~\eqref{eq:rad-hss} admits as a special case the
standard Rademacher complexity bound for fixed hypothesis sets
\citep{KoltchinskiiPanchenko2002,BartlettMendelson2002}: in that case,
we have $\cH_S = \sH$ for some $\sH$, thus $\R^\diamond_m(\cG)$
coincides with the standard Rademacher complexity $\R_m(\cG)$;
furthermore, the family of hypothesis sets is $0$-stable, thus the
bound holds with $\beta = 0$. 

Bounds~\eqref{eq:dp-hss} and \eqref{eq:induction-hss} specialize to the bounds of \citet{FeldmanVondrak2018} and \citet{FeldmanVondrak2019} respectively for the special case of standard uniform stability of algorithms, since in that case, $\sH_S$ is reduced to a singleton, $\sH_S = \set{h_S}$, and so $\supdiam = 0$, which implies that $\supcvstab \leq \supdiam + \beta = \beta$.

The Rademacher complexity-based bound~\eqref{eq:rad-hss} typically
gives the tightest control on generalization error compared to the
bounds~\eqref{eq:dp-hss} and \eqref{eq:induction-hss}, which rely on
the cruder diameter notion. However the diameter may be easier to
bound for some applications than the Rademacher complexity. To
compare the diameter-based bounds, in applications where
$\maxdiam = O(\supdiam)$, bound \eqref{eq:induction-hss} may be tighter
than \eqref{eq:dp-hss}. But, in several applications, we have
$\beta = O(\frac{1}{m})$, and then bound \eqref{eq:dp-hss} is tighter.

\ignore{
\section{Differential privacy-based bound for stable data-dependent
  hypothesis sets}
\label{sec:hss-diff}

In this section, we use recent techniques introduced in the
differential privacy literature to derive improved generalization
guarantees for stable data-dependent hypothesis sets
\citep{SteinkeUllman2017,BassilyNissimSmithSteinkeStemmerUllman2016}
(see also \citep{McSherryTalwar2007}). Our proofs also benefit from
the recent improved stability results of \cite{FeldmanVondrak2018}.
We will make use of the following lemma due to \citet[Lemma
1.2]{SteinkeUllman2017}, which reduces the task of deriving
a concentration inequality to that of upper bounding an
expectation of a maximum.

\begin{lemma}
\label{lemma:SU}
Fix $p \geq 1$. Let $X$ be a random variable with probability
distribution $\sD$ and $X_1, \ldots, X_p$ independent copies of
$X$. Then, the following inequality holds:
\[
\Pr_{X \sim \sD} \bigg[X \geq 2 \E_{X_k \sim \sD} \Big[\max \set[\big]{0, X_1,
\ldots, X_p} \Big] \bigg] \leq \frac{\log 2}{p}.
\]
\end{lemma}

We will also use the following result which, under a sensitivity
assumption, further reduces the task of upper bounding the expectation
of the maximum to that of bounding a more favorable expression. The \emph{sensitivity} of a function $f\colon \sZ^m \to \Rset$ is $\sup_{S, S' \in \sZ^{m},\ |S \cap S'| = m-1} |f(S) - f(S')|$.

\begin{lemma}[\citep{McSherryTalwar2007,
BassilyNissimSmithSteinkeStemmerUllman2016,FeldmanVondrak2018}]
\label{lemma:diff-stab}
Let $f_1, \ldots, f_p\colon \sZ^m \to \Rset$ be $p$ scoring functions
with sensitivity $\supdiam$. Let $\cA$ be the algorithm that, given a
dataset $S \in \sZ^m$ and a parameter $\e > 0$, returns the index
$k \in [p]$ with probability proportional to
$e^{\frac{\e f_k(S) }{2 \supdiam}}$. Then, $\cA$ is $\e$-differentially
private and, for any $S \in \sZ^m$, the following inequality holds:
\[
\max_{k \in [p]} \set[\big]{f_k(S)}
\leq \E_{\substack{k = \cA(\sfS)}}  \big[ f_k(S) \big] + \frac{2 \supdiam}{\e}
\log p.
\]
\end{lemma}
Notice that, if we define $f_{p + 1} = 0$, then, by the same result,
the algorithm $\cA$ returning the index $k \in [p + 1]$ with
probability proportional to
$e^{\frac{\e f_k(S) 1_{k \neq (p + 1)} }{2 \supdiam}}$ is
$\e$-differentially private and the following inequality holds for any
$S \in \sZ^m$:
\begin{equation}
\label{eq:diff}
\max\set[\Big]{0, \max_{k \in [p]} \set[\big]{f_k(S)}}
= \max_{k \in [p + 1]} \set[\big]{f_k(S)}
\leq \E_{\substack{k = \cA(\sfS)}}  \big[ f_k(S) \big] + \frac{2 \supdiam}{\e}
\log (p + 1).
\end{equation}

\begin{theorem}
\label{th:hss-diff}
Let $\cH = (\sH_S)_{S \in \sZ^m}$ be a $\beta$-stable family of
data-dependent hypothesis sets with $\supcvstab$ CV-stability. Let $\cG$ be defined as in \eqref{eq:G-def}. Then, for any $\delta \in (0, 1)$,
with probability at least $1 - \delta$ over the draw of a sample
$S \sim \sZ^m$, the following inequality holds for all $h \in \sH_S$:
\begin{equation*}
R(h) \leq \h R_S(h) + 
\min \set[\Bigg]{2 \R^\diamond_m(\cG) + [1 + 2 \beta m] \sqrt{\frac{\log
    \frac{2}{\delta}}{2m}},
\sqrt{e} \supcvstab +
4 \sqrt{\left[\tfrac{1}{m} + 2 \beta \right] \log \left[
    \tfrac{6}{\delta} \right]}}.
\end{equation*}
\end{theorem}
\begin{proof}
  For any two samples $S, S'$ of size $m$, define $\Psi(S, S')$ as follows:
\begin{equation*}
  \Psi(S, S') = \sup_{h \in \sH_{S}} R(h) - \h R_{S'}(h) .
\end{equation*}
The proof consists of deriving a high-probability bound for
$\Psi(S, S)$. To do so, by Lemma~\ref{lemma:SU} applied to the random
variable $X = \Psi(S, S)$, it suffices to bound
$\E_{\substack{\sfS \sim \sD^{pm}}} \Big[ \max\set[\Big]{0, \max_{k
    \in [p]} \set[\big]{\Psi(S_k, S_k)}} \Big]$, where
$\sfS = (S_1, \ldots, S_p)$ with $S_k$, $k \in [p]$, independent
samples of size $m$ drawn from $\sD^m$.
To bound that expectation, we use Lemma~\ref{lemma:diff-stab} and
instead bound
$\E_{\substack{\sfS \sim \sD^{pm}\\ k = \cA(\sfS)}} [\Psi(S_k,
S_k)]$, where $\cA$ is an $\e$-differentially private algorithm.
To apply Lemma~\ref{lemma:diff-stab}, we first show that, for any
$k \in [p]$, the function $f_k \colon \sfS \to \Psi(S_k, S_k)$ is
$\supdiam$-sensitive with $\supdiam = \frac{1}{m} + 2\beta$. 
Lemma~\ref{lem:supcvstab-bound} helps us express our upper
bound in terms of the CV stability coefficient $\supcvstab$.
The full proof is given in Appendix~\ref{app:hss-diff}.
\end{proof}
The hypothesis set-stability bound of this theorem admits the same
favorable dependency on the stability parameter $\beta$ as the best
existing bounds for uniform-stability recently presented by
\citet{FeldmanVondrak2018}. As with Theorem~\ref{th:hss}, the bound of
Theorem~\ref{th:hss-diff} admits as special cases both standard
Rademacher complexity bounds ($\sH = \sH$ for some fixed $\sH$ and
$\beta = 0$) and uniform-stability bounds ($\sH_S = \set{h_S}$).  In
the latter case, our bound coincides with that of
\citet{FeldmanVondrak2018} modulo constants that could be chosen to be
the same for both results.\footnote{The differences in constant terms
  are due to slightly difference choices of the parameters and a
  slightly different upper bound in our case where $e^\e$ multiplies
  the stability and the diameter, while the paper of
  \citet{FeldmanVondrak2018} does not seem to have that factor.}
Notice that the current bounds for standard uniform stability may not
be optimal since no matching lower bound is known yet
\citep{FeldmanVondrak2018}. It is very likely, however, that improved
techniques used for deriving more refined algorithmic stability bounds
could also be used to improve our hypothesis set stability guarantees.
In Appendix~\ref{app:hss-diff2}, we give an alternative version of
Theorem~\ref{th:hss-diff} with a proof technique only making use of
recent methods from the differential privacy literature, including
to derive a Rademacher complexity bound.
It might be possible to achieve a better dependency on $\beta$ for the
term in the bound containing the Rademacher complexity.  In
Appendix~\ref{app:finer}, we initiate such an analysis by deriving a
finer analysis on the expectation
$\E_{\substack{\sfS \sim \sD^{pm}}} \Big[ \max\set[\Big]{0, \max_{k
    \in [p]} \set[\big]{\Psi(S_k, S_k)}} \Big]$.
} % ends ignore for section on DP-based bounds.

\section{Applications}
\label{sec:applications}

We now discuss several applications of our learning guarantees, with some others in Appendix~\ref{app:other-applications}.
% As already mentioned, both the standard setting of a fixed
% hypothesis set $\sH_S$ not varying with $S$, that is that of
% standard generalization bounds, and the uniform stability setting
% where $\sH_S = \set{h_S}$, are special cases benefitting from our
% learning guarantees.

\subsection{Bagging}
\label{sec:bagging}

\emph{Bagging} \citep{Breiman1996} is a prominent ensemble method used
to improve the stability of learning algorithms.
It consists of generating $k$ new samples $B_1, B_2, \ldots, B_k$,
each of size $p$, by sampling uniformly with replacement from the
original sample $S$ of size $m$. An algorithm $\cA$ is then trained on
each of these samples to generate $k$ predictors $\cA(B_i)$,
$i \in [k]$. In regression, the predictors are combined by taking a
convex combination $\sum_{i = 1}^k w_i \cA(B_i)$. Here, we analyze a
common instance of bagging to illustrate the application of our
learning guarantees: we will assume a regression setting and a uniform
sampling from $S$ \emph{without replacement}.\footnote{Sampling
  without replacement is only adopted to make the analysis more
  concise; its extension to sampling with replacement is
  straightforward.} We will also assume that the loss function is
$\mu$-Lipschitz in its first argument, that the predictions are in the
range $[0, 1]$, and that all the mixing weights $w_i$ are bounded by
$\frac{C}{k}$ for some constant $C \geq 1$, in order to ensure that no
subsample $B_i$ is overly influential in the final regressor (in
practice, a uniform mixture is typically used in bagging).

To analyze bagging, we cast it in our framework. First, to deal with
the randomness in choosing the subsamples, we can equivalently imagine
the process as choosing \emph{indices} in $[m]$ to form the subsamples
rather than samples in $S$, and then once $S$ is drawn, the subsamples
are generated by filling in the samples at the corresponding
indexes. For any index $i \in [m]$, the chance that it is picked in
any subsample is $\frac{p}{m}$. Thus, by Chernoff's bound, with
probability at least $1 - \delta$, no index in $[m]$ appears in more
than $t := \frac{kp}{m} + \sqrt{\frac{2kp\log(\frac{m}{\delta})}{m}}$
subsamples. In the following, we condition on the random seed of the
bagging algorithm so that this is indeed the case, and later use a
union bound to control the chance that the chosen random seed does not
satisfy this property, as elucidated in
section~\ref{sec:randomized-algs}.

Define the data-dependent family of hypothesis sets $\cH$ as
$\sH_S := \set[\big]{\sum_{i=1}^k w_i \cA(B_i)\colon \, w \in
  \Delta_k^{C/k}}$, where $\Delta^{C/k}_k$ denotes the simplex of
distributions over $k$ items with all weights $w_i \leq
\frac{C}{k}$. Next, we give upper bounds on the hypothesis set
stability and the Rademacher complexity of $\cH$.
Assume that algorithm $\cA$ admits uniform stability $\beta_A$
\citep{BousquetElisseeff2002}, i.e.  for any two samples $B$ and $B'$
of size $p$ that differ in exactly one data point and for all
$x \in \cX$, we have $|\cA(B)(x) - \cA(B')(x)| \leq \beta_{\cA}$.
Now, let $S$ and $S'$ be two samples of size $m$ differing by one
point at the same index, $z \in S$ and $z' \in S'$. Then, consider the
subsets $B'_i$ of $S'$ which are obtained from the $B_i$s by copying
over all the elements except $z$, and replacing all instances of $z$
by $z'$. For any $B_i$, if $z \notin B_i$, then $\cA(B_i) = \cA(B'_i)$
and, if $z \in B_i$, then
$|\cA(B_i)(x) - \cA(B'_i)(x)| \leq \beta_{\cA}$ for any $x \in
\cX$. We can now bound the hypothesis set uniform stability as
follows: since $L$ is $\mu$-Lipschitz in the prediction, for any
$z'' \in \cZ$, and any $w \in \Delta^{C/k}_k$, we have
\[ 
\Big |L(\textstyle\sum_{i=1}^k w_i\cA(B_i), z'') -
  L(\textstyle\sum_{i=1}^k w_i\cA(B'_i), z'') \Big| \leq
\Big[ \frac{p}{m} +
  \sqrt{\frac{2p\log(\frac{1}{\delta})}{km}} \Big] \cdot
C\mu\beta_{\cA}.
\]
% Thus, bagging is $\beta_{\textrm{bagging}}$-uniformly stable.

It is easy to check the CV-stability and diameter of the hypothesis sets is $\Omega(1)$ in the worst case. Thus, the CV-stability-based bound \eqref{eq:dp-hss} and standard uniform-stability bound \eqref{eq:induction-hss} are not informative here, and we use the Rademacher complexity based bound \eqref{eq:rad-hss} instead. Bounding the Rademacher complexity $\h \R_S(\sH_{S, T})$ for $S, T \in \cZ^m$ is non-trivial. Instead, we can derive a reasonable upper bound by analyzing the Rademacher complexity of a larger function class. Specifically, for any $z \in \cZ$, define the
$d := {2m \choose p}$-dimensional vector
$u_z = \langle \cA(B)(z) \rangle_{B \subseteq S \cup T, |B| =
  p}$. Then, the class of functions is
$\sF_{S, T} := \{z \mapsto w^\top u_z:\ w \in \mathbb{R}^d,\ \|w\|_1 =
1\}$. Clearly $\sH_{S, T} \subseteq \sF_{S, T}$. Since
$\|u_z\|_\infty \leq 1$, a standard Rademacher complexity bound (see
Theorem 11.15 in \citep{MohriRostamizadehTalwalkar2012}) implies
$\h \R_S(\sF_{S, T}) \leq \sqrt{\frac{2 \log \big( 2\binom{2m}{p}
    \big)}{m}} \leq \sqrt{\frac{2p\log(4m)}{m}}$. Thus, by Talagrand's
inequality, we conclude that
$\h \R_S(\sG_{S, T}) \leq \mu\sqrt{\frac{2p\log(4m)}{m}}$.
In view of that, by Theorem~\ref{th:hss}, for any $\delta > 0$,
with probability at least $1 - 2\delta$ over the draws of a sample
$S \sim \sD^m$ and the randomness in the bagging algorithm, the
following inequality holds for any $h \in \sH_S$:
\[
R(h) \leq \h R_S(h) + 
2\mu\sqrt{\frac{2p\log(4m)}{m}} + \Bigg[ 1 + 2 \Bigg[ p +
  \sqrt{\frac{2pm\log(\frac{1}{\delta})}{k}} \Bigg] \cdot
C\mu\beta_{\cA} \Bigg] \sqrt{\frac{\log
    \frac{2}{\delta}}{2m}}.
\]
For $p = o(\sqrt{m})$
and $k = \omega(p)$, the generalization gap goes to $0$ as
$m \to \infty$, \emph{regardless} of the stability of $\cA$. This
gives a new generalization guarantee for bagging, similar (but incomparable) to the one derived by \citet{ElisseeffEvgeniouPontil2005}. One major point of difference is that unlike their bound, our bound allows for non-uniform averaging schemes.

\subsection{Stochastic strongly-convex optimization}
\label{sec:sco}

Here, we consider data-dependent hypothesis sets based on stochastic
strongly-convex optimization algorithms. As shown by
\citet{Shalev-ShwartzShamirSrebroSridharan2010}, uniform convergence
bounds do not hold for the stochastic convex optimization problem in
general. 
% As a result, the data-dependent hypothesis sets we will
% define cannot be analyzed using standard tools for deriving
% generalization bounds. However, using arguments based on our notion of
% hypothesis set stability, we can provide learning guarantees here.

Consider $K$ stochastic strongly-convex optimization algorithms $\cA_j$, each returning vector $\h w^S_j$, after receiving sample $S \in \sZ^m$, $j \in
[K]$. As shown by \citet{Shalev-ShwartzShamirSrebroSridharan2010}, such algorithms are $\beta = O(\frac{1}{m})$ sensitive in their output vector, i.e. 
for all $j \in [K]$, we have $\| \h w^S_j - \h w^{S'}_j \| \leq \beta$ if $S$ and $S'$ differ by one point. 
% We will also assume that these vectors are bounded by some
% $D > 0$ that is $\| \h w^S_j \|_2 \leq D$, for all $j \in [K]$. This
% can be shown to be the case, for example, for algorithms based on
% empirical risk minimization with a strongly convex regularization term
% with $\beta = O(\frac{1}{m})$
% .

Assume that the loss $L(w, z)$ is $\mu$-Lipschitz with respect to its
first argument $w$. Let the data-dependent hypothesis set be defined
as follows: $\sH_S = \left\{\sum_{j = 1}^K \alpha_j \h w^S_j \colon \ \alpha \in \Delta_K \cap \sfB_1(\alpha_0, r)\right\}$, where $\alpha_0$ is in the simplex of distributions $\Delta_K$ and $\sfB_1(\alpha_0, r)$ is the $L_1$ ball of radius $r > 0$ around $\alpha_0$. We choose $r = \frac{1}{2\mu D \sqrt{m}}$. A natural choice for $\alpha_0$ would be the uniform mixture.

Since the loss function is $\mu$-Lipschitz, the family of hypotheses
$\sH_S$ is $\mu \beta$-stable. In this setting, bounding the Rademacher complexity is difficult, so we resort to the diameter based bound \eqref{eq:dp-hss} instead. Note that for any
$\alpha, \alpha' \in \Delta_K \cap \sfB_1(\alpha_0, r)$ and any
$z \in \sZ$, we have
\begin{align*}
L\bigg(\sum_{j = 1}^K \alpha_j \h w^S_j, z \bigg) - L \bigg ( \sum_{j = 1}^K \alpha'_j \h
w^S_j, z \bigg)
& \leq \mu \bigg\| \sum_{j = 1}^K (\alpha_i - \alpha'_j) \h w^S_j
  \bigg\|_2
\leq \mu \big\| [w^S_1 \, \cdots \, w^S_K] \big\|_{1, 2} \, \| \alpha - \alpha' \|_1 
\leq 2\mu r D,
\end{align*}
where 
% $\big\| [w^S_1 \, \cdots \, w^S_K] \big\|_{1, 2}$ is the
% subordinate norm of matrix 
% $[w^S_1 \, \cdots \, w^S_K]$
% defined by 
$\big\| [w^S_1 \, \cdots \, w^S_K] \big\|_{1, 2} := \max_{x \neq 0} \frac{\| \sum_{j = 1}^k x_j w^S_j \|_2}{\| x \|_1} = \max_{i \in [K]} \|w^S_i\|_2 \leq D$.
Thus, the average diameter admits the following upper bound:
$\h \Delta \leq 2\mu r D = \frac{1}{\sqrt{m}}$. In view of that, by
Theorem~\ref{th:hss}, for any $\delta > 0$, with probability at
least $1 - \delta$, the following holds for all
$\alpha \in \Delta_K \cap \sfB_1(\alpha_0, r)$:
\begin{align*}
\E_{z \sim \sD} \bigg[ L \bigg( \sum_{j = 1}^K \alpha_j \h w^S_j, z \bigg) \bigg]
& \leq \frac{1}{m} \sum_{i = 1}^m L \bigg( \sum_{j = 1}^K \alpha_i \h
w^S_j, z^S_i \bigg) + \sqrt{\frac{e}{m}}
+ \sqrt{e} \mu \beta +
4 \sqrt{\left(\frac{1}{m} + 2 \mu \beta \right) \log \left(
    \frac{6}{\delta} \right)}.
\end{align*}
The second stage of an algorithm in this context consists of choosing
$\alpha$, potentially using a non-stable algorithm.
This application both illustrates the use of our learning
bounds using the diameter and its application even in
the absence of uniform convergence bounds.

As an aside, we note that the analysis of section~\ref{sec:bagging} can be carried over to this setting, by setting $\cA$ to be a stochastic strongly-convex optimization algorithm which outputs a weight vector $\hat{w}$. This yields generalization bounds for aggregating over a larger set of mixing weights, albeit with the restriction that each algorithm uses only a small part of $S$.

\subsection{$\supdiam$-sensitive feature mappings}
\label{sec:feature-mappings}

Consider the scenario where the training sample $S \in \sZ^m$ is used
to learn a non-linear feature mapping $\Phi_S\colon \sX \to \Rset^N$
that is $\supdiam$-sensitive for some $\supdiam =
O(\frac{1}{m})$. $\Phi_S$ may be the feature mapping corresponding
to some positive definite symmetric kernel or a mapping defined by
the top layer of an artificial neural network trained on $S$, with
a stability property.

To define the second state, let $\sL$ be a set of $\gamma$-Lipschitz functions $f\colon \Rset^N \to \Rset$. Then we define
$\sH_S = \set[\big]{x \mapsto f(\Phi_S(x))\colon f \in \sL}$.
Assume that the loss function $\ell$ is $\mu$-Lipschitz with respect
to its first argument. Then, for any hypothesis
$h\colon x \mapsto f(\Phi_S(x)) \in \sH_S$ and any sample $S'$
differing from $S$ by one element, the hypothesis
$h'\colon x \mapsto f(\Phi_{S'}(x)) \in \sH_{S'}$ admits losses
that are $\beta$-close to those of $h$, with
$\beta = \mu \gamma \supdiam$, since, for all
$(x, y) \in \sX \times \sY$, by the Cauchy-Schwarz inequality, the
following inequality holds:
\begin{align*}
\ell(f(\Phi_S(x)), y) - \ell(f(\Phi_{S'}(x)), y)
\leq \mu |f((\Phi_S(x)) - f(\Phi_{S'}(x))|
\leq \mu \gamma \| \Phi_S(x) - \Phi_{S'}(x) \|
\leq \mu \gamma \supdiam.
\end{align*}
Thus, the family of hypothesis set $\cH = (\sH_S)_{S \in \sZ^m}$ is
uniformly $\beta$-stable with
$\beta = \mu \gamma \supdiam = O(\frac{1}{m})$.  In view of that, by
Theorem~\ref{th:hss}, for any $\delta > 0$, with probability at
least $1 - \delta$ over the draw of a sample $S \sim \sD^m$, the
following inequality holds for any $h \in \sH_S$:
\begin{equation}
\label{eq:delta-sensitive-bound}
R(h) \leq \h R_S(h) + 
2 \R^\diamond_m(\cG) + (1 + 2  \mu \gamma \supdiam m) \sqrt{\tfrac{1}{2m}\log(\tfrac{1}{\delta})}.
\end{equation}
Notice that this bound applies even when the second stage of an
algorithm, which consists of selecting a hypothesis $h_S$ in $\sH_S$,
is not stable. A standard uniform stability guarantee cannot be used
in that case. The setting described here can be straightforwardly
extended to the case of other norms for the definition of sensitivity
and that of the norm used in the definition of $\sH_S$.

\subsection{Distillation}
\label{sec:distillation}

Here, we consider \emph{distillation algorithms} which, in the first
stage, train a very complex model on the labeled sample. Let
$f^*_S\colon \sX \to \Rset$ denote the resulting predictor for a
training sample $S$ of size $m$. We will assume that the training
algorithm is $\beta$-sensitive, that is
$ \| f^*_S - f^*_{S'} \| \leq \beta = O(\tfrac{1}{m})$ for $S$ and $S'$
differing by one point.

\begin{wrapfigure}{r}{0.4\textwidth}
%\vskip -.15in
\centering
\includegraphics[scale=.4]{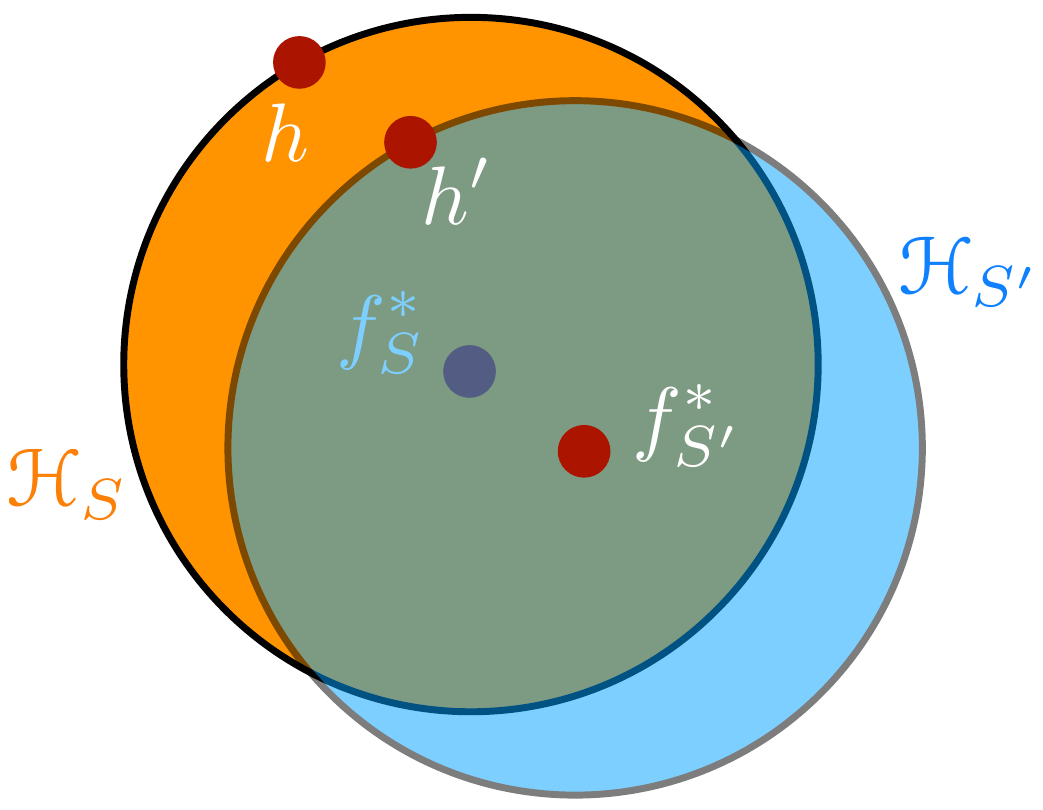}
\caption{Illustration of the distillation hypothesis sets. Notice that
  the diameter of a hypothesis set $\sH_S$ may be large here.}
\vskip -.15in
\label{fig:fig2}
\end{wrapfigure}

In the second stage, the algorithm selects a hypothesis
that is $\gamma$-close to $f^*_S$ from a less complex family of
predictors $\sH$. This defines the following sample-dependent
hypothesis set: $\sH_S = \set[\big]{h \in \sH \colon \| (h - f^*_S) \|_{\infty} \leq \gamma}$.

Assume that the loss $\ell$ is $\mu$-Lipschitz with respect to its
first argument and that $\sH$ is a subset of a vector space. Let $S$
and $S'$ be two samples differing by one point.  Note, $f^*_S$ may not
be in $\sH$, but we will assume that $f^*_{S'} - f^*_S$ is in $\sH$.
Let $h$ be in $\sH_S$, then the hypothesis $h' = h + f^*_{S'} - f^*_S$
is in $\sH_{S'}$ since
$\| h' - f^*_{S'} \|_\infty = \| h - f^*_S \|_\infty \leq \gamma$.
Figure~\ref{fig:fig2} illustrates the hypothesis sets.  By the
$\mu$-Lipschitzness of the loss, for any $z = (x, y) \in \sZ$,
$| \ell(h'(x), y) - \ell(h(x), y) | \leq \mu \| h'(x) - h(x) \|_\infty
= \mu \| f^*_{S'} - f^*_S \| \leq \mu \beta$. Thus, the family of
hypothesis sets $\sH_S$ is $\mu \beta$-stable.

In view of that, by Theorem~\ref{th:hss}, for any $\delta > 0$, with
probability at least $1 - \delta$ over the draw of a sample
$S \sim \sD^m$, the following inequality holds for any $h \in \sH_S$:
\[
  R(h) \leq \h R_S(h) + 2 \R^\diamond_m(\cG) + (1 + 2 \mu \beta m)
\sqrt{\tfrac{1}{2m}\log(\tfrac{1}{\delta})}.
\]
Notice that a standard uniform-stability argument would not
necessarily apply here since $\sH_S$ could be relatively
complex and the second stage not necessarily stable.

\ignore{
\subsection{Anti-distillation}

A similar setup is that of \emph{anti-distillation} where
the predictor $f^*_S$ in the first stage is chosen from 
a simpler family, say that of linear hypotheses, and 
where the sample-dependent hypothesis set $\sH_S$ is 
the subset of a very rich family $\sH$. 
$\sH_S$ is defined as the set of predictors that
are close to $f^*_S$:
\[
  \sH_S = \set[\big]{h \in \sH \colon (\| (h - f^*_S) \|_{\infty} \leq
    \gamma) \wedge (\| (h - f^*_S) 1_S \|_{\infty} \leq
    \supdiam_m)},
\]
with $\supdiam = O(1/\sqrt{m})$. Thus, the restrictions to $S$ of a
hypothesis $h \in \sH_S$ is close to $f^*_S$ in $\ell_\infty$-norm. As
shown in the previous section, the family of hypothesis sets $\sH_S$
is $\mu \beta_m$-stable.  However, here, the hypothesis sets $\sH_S$
could be very complex and the Rademacher complexity
$\R^\diamond_m(\cH)$ not very favorable. Nevertheless, by
Theorem~\ref{th:hss-diff}, for any $\delta > 0$, with probability at
least $1 - \delta$ over the draw of a sample $S \sim \sD^m$, the
following inequality holds for any $h \in \sH_S$:
\[
R(h) \leq \h R_S(h) + 
\sqrt{e} \mu (\supdiam_m + \beta_m) +
4 \sqrt{\left[\tfrac{1}{m} + 2 \mu \beta_m \right] \log \left[
    \tfrac{6}{\delta} \right]}.
\]  
Notice that a standard uniform-stability does not apply
here since the $(1/\sqrt{m})$-closeness of the hypotheses to $f^*_S$
on $S$ does not imply their global $(1/\sqrt{m})$-closeness.
}

\ignore{
\subsection{Principal Components Regression}

Principal Components Regression is a very commonly used technique in
data analysis. In this setting, $\sX \subseteq \Rset^d$ and
$\sY \subseteq \Rset$, with a loss function $\ell$ that is
$\mu$-Lipschitz in the prediction. Given a sample
$S = \set{(x_i, y_i) \in \sX \times \sY\colon i \in [m] }$, we learn a
linear regressor on the data projected on the principal
$k$-dimensional space of the data. Specifically, let
$\Pi_S \in \Rset^{d \times d}$ be the projection matrix giving the
projection of $\Rset^d$ onto the principal $k$-dimensional subspace
of the data, i.e.\ the subspace spanned by the top $k$ left singular
vectors of the design matrix $X_S = [x_1, x_2, \cdots, x_m]$. The
hypothesis space $\sH_S$ is then defined as
$\sH_S = \set{x \mapsto w^\top \Pi_Sx \colon w \in \Rset^k,\ \|w\| \leq
\gamma }$, where $\gamma$ is a predefined bound on the norm of the
weight vector for the linear regressor. Thus, this can be seen as an
instance of the setting in section~\ref{sec:feature-mappings}, where
the feature mapping $\Phi_S$ is defined as $\Phi_S(x) = \Pi_Sx$.

To prove generalization bounds for this setup, we need to show that
these feature mappings are stable. To do that, we make the following
assumptions:
\begin{enumerate}

\item For all $x \in \sX$, $\|x\| \leq r$ for some constant
  $r \geq 1$.

\item The data covariance matrix $\E_x[xx^\top]$ has a gap of
  $\lambda > 0$ between the $k$-th and $(k+1)$-th largest eigenvalues.

\end{enumerate}
The matrix concentration bound of \citet{RudelsonVershynin2007}
implies that with probability at least $1-\delta$ over the choice of
$S$, we have
$\| X_SX_S^\top - m\E_{x}[xx^\top] \| \leq cr^2 \sqrt{m \log(m)
  \log(\tfrac{2}{\delta})}$ for some constant $c > 0$. Suppose $m$ is
large enough so that
$cr^2 \sqrt{m \log(m) \log(\tfrac{2}{\delta})} \leq
\frac{\lambda}{2}m$. Then, the gap between the $k$-th
and $(k + 1)$-th largest eigenvalues of $X_SX_S^\top$ is at least
$\frac{\lambda}{2}m$. Now, consider changing one sample point
$(x, y) \in S$ to $(x, y')$ to produce the sample set $S'$. Then, we
have $X_{S'}X_{S'}^\top = X_SX_S^\top - xx^\top + x'x^{'\top}$. Since
$\| - xx^\top + x'x^{'\top} \| \leq 2r^2$, by standard matrix
perturbation theory bounds \citep{Stewart1998}, we have
$\|\Pi_S - \Pi_{S'}\| \leq O(\frac{r^2}{\lambda m})$. Thus,
$\|\Phi_S(x) - \Phi_{S'}(x)\| \leq \| \Pi_S - \Pi_{S'} \| \|x\| \leq
O(\frac{r^3}{\lambda m})$.

Now, to apply the bound of \eqref{eq:delta-sensitive-bound}, we need
to compute a suitable bound on $\R^\diamond_m(\cH)$. For this, we
apply Lemma~\ref{lemma:linear-rad}. For any $\|w\| \leq \gamma$, since
$\|\Pi_S\| = 1$, we have $\|\Pi_S w\| \leq \gamma$. So the hypothesis
set
$\sH'_S = \{x \mapsto w^\top \Pi_Sx\colon w \in \Rset^k,\ \|\Pi_Sw\|
\leq \gamma\}$ contains $\sH_S$. By Lemma~\ref{lemma:linear-rad}, we
have $\R^\diamond_m(\cH') \leq \frac{\gamma r}{\sqrt{m}}$. Thus, by
plugging the bounds obtained above in
\eqref{eq:delta-sensitive-bound}, we conclude that with probability at
least $1-2\delta$ over the choice of $S$, for any $h \in \sH_S$, we
have
\[
R(h) \leq \h R_S(h) + 
O\left(\mu \gamma \frac{r^3}{\lambda}\sqrt{\frac{\log
    \frac{1}{\delta}}{m}}\right).
\]
}

\ignore{
\subsection{Interpolation methods}

Consider the following instance of interpolation methods. Data consist
of pairs $z = (x, y) \in \sX \times [-1, 1]$, where $\sX$ is the input
domain. We want to learn predictors to minimize square loss, i.e. the
loss for predicting $\hat{y}$ for an example $(x, y)$ is
$\frac{1}{4}(\hat{y}_{\text{trunc}} - y)^2$, where
$\hat{y}_{\text{trunc}} = \max\{\min\{\hat{y}, 1\}, -1\}$. Now suppose
we have a family $\Phi$ of feature mappings
$\phi\colon \sX \to \Rset^d$. The dimension $d$ of the range of
different mappings may be different.

For any sample set
$S = \{(x_i, y_i) \in \sX \times [-1, 1]\colon i \in [m]\}$, we define
a hypothesis set $\sH_S$ based on all feature mappings $\phi \in \Phi$
which admit an \emph{interpolating linear regressor}, i.e. those
$\phi$ for which there is a linear regressor with weights $w$ such
that $w^\top \phi(x_i) = y_i$ for all $i \in [m]$ (naturally, such
regressors have 0 loss). Let $\Phi_S$ denote the set of all
$\phi \in \Phi$ which admit an interpolating linear regressor on
$S$. For each $\phi \in \Phi_S$, are interested in the linear
regressor with zero loss whose weights have minimum norm: these
weights are $w_{\phi, S} = X_\phi^\dagger Y$, where
$X_\phi = [\phi(x_1), \phi(x_2), \ldots, \phi(x_m)]$ and
$Y = [y_1, y_2, \ldots, y_m]^\top$, and $\dagger$ denotes the
Moore-Penrose pseudo-inverse. We now define the hypothesis set $\sH_S$
as follows:
$\sH_S = \{ x \mapsto w_{\phi, S}^\top \phi(x):\ \phi \in
\Phi_S\}$. We would like to provide generalization bounds for this
family of data-dependent hypothesis sets. For this, we need to make
some simplifying assumptions on the data: for all $\phi \in \Phi$,
\begin{enumerate}

\item there are constants $D, r > 0$ such that $\|\phi(x)\| \leq D$ for all $x \in \sX$ and $\|\E_x[\phi(x)]\| \geq r$, and

\item the smallest non-zero eigenvalue of the data covariance matrix
  $\E_{x}[\phi(x)\phi(x)^\top]$ is at least $\lambda$, for some
  constant $\lambda > 0$.
\end{enumerate}
The second assumption might seem a bit unnatural, but it is easy to
ensure that it holds if we have access to a large amount of unlabeled
data by adding an extra PCA step on top of the feature mappings where
all eigenvalues smaller than $\lambda$ are removed.

Under these assumptions, we can now bound the uniform stability as
follows. Instead of bounding the change-one uniform stability, we will
bound the leave-one-out uniform stability. So let $S'$ be a sample set
obtained from $S$ by removing sample point, say $(x, y)$. We will show
that with high probability,
$\|w_{\phi, S} - w_{\phi, S'}\| = O(\frac{1}{m})$.

To do this, we can rewrite the formula for $w_{\phi, S}$ as
$w_{\phi, S} = (X_\phi X_\phi^\top)^\dagger X_\phi Y$, and similarly,
$w_{\phi, S'} = (X'_\phi X_\phi^{'\top})^\dagger X'_\phi Y'$, where
$X'_{\phi}$ and $Y'$ are $X_\phi$ and $Y$ with $\phi(x)$ and $y$
removed, respectively. Thus,
$X'_\phi X_\phi^{'\top} = X_\phi X_\phi^\top - \phi(x)\phi(x)^\top$,
and so, by the Sherman-Morrison-Woodbury formula, we have
\[(X'_\phi X_\phi^{'\top})^\dagger =  (X_\phi X_\phi^\top)^\dagger + \frac{(X_\phi X_\phi^\top)^\dagger \phi(x)\phi(x)^\top (X_\phi X_\phi^\top)^\dagger}{1 - \phi(x)^\top (X_\phi X_\phi^\top)^\dagger\phi(x)}. \]
Furthermore, we have $X'_\phi Y' = X_\phi Y - \phi(x)y$. Thus,
\begin{equation}
\label{eq:w_phi_stability}
  w_{\phi, S} - w_{\phi, S'} = (X_\phi X_\phi^\top)^\dagger \phi(x)y -\frac{(X_\phi X_\phi^\top)^\dagger \phi(x)\phi(x)^\top (X_\phi X_\phi^\top)^\dagger X'_\phi Y'}{1 - \phi(x)^\top (X_\phi X_\phi^\top)^\dagger\phi(x)} 
\end{equation} 
By standard matrix concentration bounds
\citep{RudelsonVershynin2007}, with probability at least
$1-\delta$ over the choice of $S$, we have
$\|X_\phi X_\phi^\top - m\E_{x}[\phi(x)\phi(x)^\top]\| \leq
\frac{1}{c}D^2 \sqrt{m \log(m) \log(\tfrac{2}{\delta})}$. Suppose $m$
is large enough so that
$\frac{1}{c}D^2 \sqrt{m \log(m) \log(\tfrac{2}{\delta})} \leq
\frac{m}{2}\lambda$. Using this bound, it is easy to check that
$\|(X_\phi X_\phi^\top)^\dagger\phi(x)\| \leq \frac{2D}{\lambda m}$
and $\|X'_\phi Y'\| \leq mD$. Plugging in these bounds in
\eqref{eq:w_phi_stability}, we conclude that
$\|w_{\phi, S} - w_{\phi, S'}\| = O(\frac{D^3}{\lambda^2 m})$. Since
$\|\phi(x)\| \leq D$ and the loss function is $1$-Lipschitz, we
conclude that the uniform stability is bounded by
$O(\frac{D^4}{\lambda^2 m})$. {\bf SATYEN: there is a bug here. For
  hypothesis sets, leave-one-out stability does not bound change-one
  stability.}

Next, the diameter of $\cH$ is $0$ since the regressors interpolate
the training set. Furthermore, for all $h \in \sH_S$, $\h R_S(h) = 0$
for the same reason. Thus, by Theorem~\ref{th:hss-diff}, we get the
following generalization bound: with probability at least $1-2\delta$
over the choice of $S$, for all $h \in \sH_S$, we have
\[
R(h) = O\left(\frac{D^4}{\lambda^2 m} +
    \sqrt{\frac{D^4\log(\tfrac{1}{\delta})}{\lambda^2 m}}\right).
\]
{\bf SATYEN: should we omit the first term above? The bound is vacuous
  when the term is larger than 1 anyway.}
}

\section{Conclusion}

We presented a broad study of generalization with data-dependent
hypothesis sets, including general learning bounds using a notion of
transductive Rademacher complexity and, more importantly, learning
bounds for stable data-dependent hypothesis sets. We illustrated the
applications of these guarantees to the analysis of several
problems. Our framework is general and covers learning scenarios
commonly arising in applications for which standard generalization
bounds are not applicable. Our results can be further augmented and
refined to include model selection bounds and local Rademacher
complexity bounds for stable data-dependent hypothesis sets (to be
presented in a more extended version of this manuscript), and further
extensions described in Appendix~\ref{app:extensions}.  Our analysis
can also be extended to the non-i.i.d.\ setting and other learning
scenarios such as that of transduction.  Several by-products of our
analysis, including our proof techniques, new guarantees for
transductive learning, and our PAC-Bayesian bounds for randomized
algorithms, both in the sample-independent and sample-dependent cases,
can be of independent interest. While we highlighted several
applications of our learning bounds, a tighter analysis might be
needed to derive guarantees for a wider range of data-dependent
hypothesis classes or scenarios.

\ignore{
Improving the dependency on the
diameter, as well as coming up with more refined versions of the
notion of Rademacher complexity for data-dependent hypothesis sets
defined in this paper are a few avenues for further research.
}

\paragraph{Acknowledgements.}
HL is supported by NSF IIS-1755781. The work of SG and MM was partly
supported by NSF CCF-1535987, NSF IIS-1618662, and a Google Research
Award. KS would like to acknowledge NSF
CAREER Award 1750575 and Sloan Research Fellowship.

\bibliographystyle{plainnat}
\bibliography{hss}

\newpage
\appendix

\newpage
\section{Further background on stability}
\label{app:stability}

The study of stability dates back to early work on the analysis of
$k$-neareast neighbor and other local discrimination rules
\citep{Rogers1978,DevroyeWagner1979}.  Stability has been critically
used in the analysis of stochastic optimization
\citep{Shalev-ShwartzShamirSrebroSridharan2010} and online-to-batch
conversion \citep{CesaBianchiConconiGentile2001}. Stability bounds
have been generalized to the non-i.i.d.\ settings, including
stationary \citep{MohriRostamizadeh2010} and non-stationary
\citep{KuznetsovMohri2017} $\phi$-mixing and $\beta$-mixing processes.
They have also been used to derive learning bounds for transductive
inference \citep{CortesMohriPechyonyRastogi2008}.  Stability bounds
were further extended to cover \emph{almost stable} algorithms by
\cite{KutinNiyogi2002}. These authors also discussed a number of
alternative definitions of stability, see also
\citep{KearnsRon1997}. An alternative notion of stability was also
used by \cite{KaleKumarVassilvitskii2011} to analyze $k$-fold
cross-validation for a number of stable algorithms.

\newpage
\section{Properties of data-dependent Rademacher complexity}
\label{app:rademacher}

In this section, we highlight several key properties of our notion of
data-dependent Rademacher complexity.

\ignore{
\subsection{Contraction lemma}

\begin{lemma}
\label{lemma:contraction}
Let $\Phi_1, \ldots, \Phi_m$ be $\mu$-Lipschitz functions from $\Rset$
to $\Rset$ and $\sigma_1, \ldots, \sigma_m$ be Rademacher random
variables. Then, for any family $\cH = (\sH_S)_{S \in \sZ^m}$ of
data-dependent hypothesis sets of real-valued functions, the following
inequality holds:
\begin{equation*}
  \frac{1}{m} \E_\bsigma \bigg[ \sup_{h \in \sH^\bsigma_{S, T}}
  \sum_{i = 1}^m
    \sigma_i (\Phi_i \circ h)(z^T_i) \bigg]
    \leq \frac{\mu}{m} \E_\bsigma \bigg[
      \sup_{h \in \sH^\bsigma_{S, T}} \sum_{i = 1}^m\sigma_i h(z^T_i) \bigg]
      = \mu \, \h \R^\diamond_{S, T}(\sH) \,.
\end{equation*}
In particular, if $\Phi_i = \Phi$ for all $i \in [m]$, then the
following holds:
\begin{equation*}
  \h \R^\diamond_{S, T}(\Phi \circ \sH) \leq \mu \, \h \R^\diamond_{S,
    T}(\sH).
\end{equation*}
\end{lemma}
\begin{proof}
Fix samples $S = (z^S_1, \ldots, z^S_m)$ and $S = (z^T_1, \ldots,
z^T_m)$,  then, by definition,
\begin{align*}
& \frac{1}{m} \E_\bsigma \bigg[ \sup_{h \in \sH^\bsigma_{S, T}}
  \sum_{i = 1}^m
    \sigma_i (\Phi_i \circ h)(z^T_i) \bigg]\\
& = \frac{1}{m} \E_{\sigma_1, \ldots, \sigma_{m - 1}} \!\! \bigg[
  \E_{\sigma_m} \bigg[
\sup_{h \in \sH^\bsigma_{S, T}} u_{m - 1}(h)  + \sigma_m (\Phi_m \circ h)(z^T_m) \bigg]
\bigg],
\end{align*}
where $u_{m - 1}(h) = \sum_{i = 1}^{m - 1} \sigma_i (\Phi_i \circ
h)(z^T_i)$. Define $\bsigma^+$ and $\bsigma^-$ by
$\bsigma^+ = (\sigma_1, \ldots, \sigma_{m - 1}, +1)$ and
$\bsigma^- = (\sigma_1, \ldots, \sigma_{m - 1}, -1)$.  Then, by
definition of the supremum, for any $\e > 0$, there exist
$h_1 \in \sH^{\bsigma^+}_{S, T}$ and $h_2 \in \sH^{\bsigma^-}_{S, T}$
such that
\begin{align*}
& u_{m - 1}(h_1) + (\Phi_m \circ h_1)(z^T_m) 
\geq (1 - \e) \bigg[ \sup_{h \in \sH^{\bsigma^+}_{S, T}} u_{m - 1}(h)  +  
(\Phi_m \circ h)(z^T_m) \bigg]\\
\text{and} \quad
& u_{m - 1}(h_2) - (\Phi_m \circ h_2)(z^T_m)
\geq (1 - \e) \bigg[ \sup_{h \in \sH^{\bsigma^-}_{S, T}} u_{m - 1}(h) 
(\Phi_m \circ h)(z^T_m) \bigg].
\end{align*}
Thus, for any $\e > 0$, by definition of $\E_{\sigma_m}$,
\begin{align*}
& (1 - \e) \E_{\sigma_m} \bigg[
\sup_{h \in \sH^\bsigma_{S, T}} u_{m - 1}(h)  + \sigma_m (\Phi_m \circ h)(z^T_m) \bigg]\\
& = \frac{1 - \e}{2} \sup_{h \in \sH^{\bsigma^+}_{S, T}} u_{m - 1}(h) 
+ (\Phi_m \circ h)(z^T_m) + \frac{1 - \e}{2} \sup_{h \in \sH^{\bsigma^-}_{S, T}} u_{m -
  1}(h)  - (\Phi_m \circ h)(z^T_m) \\
& \leq \frac{1}{2} \bigg[ u_{m - 1}(h_1) + (\Phi_m \circ h_1)(z^T_m)
  \bigg] 
+ \frac{1}{2} \bigg[ u_{m - 1}(h_2) - (\Phi_m \circ h_2)(z^T_m) \bigg].
\end{align*}
Let $\sigma = \sgn( h_1(z^T_m) - h_2(z^T_m) )$.  Then, the previous inequality implies
\begin{align*}
& (1 - \e) \E_{\sigma_m} \bigg[
\sup_{h \in \sH^\bsigma_{S, T}} u_{m - 1}(h)  + \sigma_m (\Phi_m \circ h)(z^T_m) \bigg]\\
& \leq \frac{1}{2} \bigg[ u_{m - 1}(h_1) + u_{m - 1}(h_2) + \mu \sigma \Big[ h_1(z^T_m) -
h_2(z^T_m) \Big] \bigg] & (\text{Lipschitz property})\\
& = \frac{1}{2} \bigg[ u_{m - 1}(h_1) + \mu \sigma h_1(z^T_m) \bigg]
  + 
\frac{1}{2} \bigg[ u_{m - 1}(h_2) - \mu \sigma h_2(z^T_m) \bigg] & (\text{rearranging})\\
& \leq \frac{1}{2} \sup_{h \in \sH^{\bsigma^+}_{S, T}} \bigg[ u_{m - 1}(h) + \mu \sigma h(z^T_m) \bigg]
  + 
\frac{1}{2} \sup_{h \in \sH^{\bsigma^-}_{S, T}} \bigg[ u_{m - 1}(h) - \mu \sigma h(z^T_m) \bigg] & (\text{definition of $\sup$})\\
& = \E_{\sigma_m} \bigg[ \sup_{h \in \sH^{\bsigma}_{S, T}} u_{m - 1}(h) + \mu \sigma_m 
h(z^T_m) \bigg]. & (\text{definition of $\E_{\sigma_m}$})
\end{align*}
Since the inequality holds for all $\e > 0$, we have
\begin{equation*}
\E_{\sigma_m}\bigg[ \sup_{h \in \sH^{\bsigma}_{S, T}} u_{m - 1}(h)  + \sigma_m (\Phi_m \circ
h)(z^T_m) \bigg] \leq \E_{\sigma_m} \bigg[ \sup_{h \in \sH^{\bsigma}_{S, T}} u_{m - 1}(h)
+ \mu \sigma_m h(z^T_m) \bigg].
\end{equation*}
Proceeding in the same way for all other $\sigma_i$ ($i \neq m$) proves the
lemma.
\end{proof}
}

\subsection{Upper-bound on Rademacher complexity of data-dependent
  hypothesis sets}

\begin{lemma}
\label{lemma:rad}
For any sample $S = (x^S_1, \ldots, x^S_m) \in \Rset^N$, define the
hypothesis set $\sH_S$ as follows:
\[
\sH_S = \set[\bigg]{x \mapsto w^S \cdot x \colon \ w^S = \sum_{i = 1}^m
  \alpha_i x^S_i, \| \alpha \|_1 \leq \Lambda_1 },
\]
where $\Lambda_1 \geq 0$.
Define $r_T$ and $r_{S \cup T}$ as follows:
$r_T = \sqrt{\frac{\sum_{i = 1}^m \| x^T_i \|_2^2}{m}}$ and  $r_{S \cup T} =
\max_{x \in S \cup T} \| x \|_2$.
Then, the empirical Rademacher complexity of the family of
data-dependent hypothesis sets $\cH = (\sH_S)_{S \in \sX^m}$ can be
upper-bounded as follows:
\[
\h \R^\diamond_{S, T}(\cH) 
\leq r_T \, r_{S \cup T} \Lambda_1 \sqrt{\frac{2 \log (4m)}{m}} 
\leq r^2_{S \cup T} \Lambda_1 \sqrt{\frac{2 \log (4m)}{m}}.
\]
\end{lemma}
\begin{proof}
The following inequalities hold:
\begin{align*}
\h \R^\diamond_{S, T}(\cH) 
= \frac{1}{m} \E_{\bsigma} \bigg[\sup_{h \in \sH_{S, T}^\bsigma} \sum_{i =
  1}^m \sigma_i h(x^T_i) \bigg]
& = \frac{1}{m} \E_{\bsigma} \bigg[\sup_{\| \alpha \|_1
  \leq \Lambda_1} \sum_{i = 1}^m \sigma_i \sum_{j = 1}^m \alpha_j
  x^{S_{T,\bsigma}}_j \cdot x^T_i\bigg]\\
& = \frac{1}{m} \E_{\bsigma} \bigg[\sup_{\| \alpha \|_1
  \leq \Lambda_1} \sum_{j = 1}^m \alpha_j \left(x^{S_{T,\bsigma}}_j \sum_{i = 1}^m \sigma_i 
   \cdot x^T_i \right)\bigg]\\
& = \frac{\Lambda_1}{m} \E_{\bsigma} \bigg[\max_{j \in [m]} \left| x^{S_{T,\bsigma}}_j
  \cdot \sum_{i = 1}^m \sigma_i x^T_i \right| \bigg]\\
& \leq \frac{\Lambda_1}{m} \E_{\bsigma} \bigg[\max_{\substack{x' \in S \cup T \\ \sigma' \in \{-1,+1\}}} 
  \sum_{i = 1}^m \sigma_i (\sigma'x' \cdot x^T_i) \bigg].
\end{align*}
The norm of the vector $z' \in \Rset^m$ with coordinates
$(\sigma'x' \cdot x^T_i)$ can be bounded as follows:
\[
  \sqrt{\sum_{i = 1}^m (\sigma'x' \cdot x^T_i)^2} \leq \| x' \| \sqrt{\sum_{i
      = 1}^m \| x^T_i \|^2} \leq r_{S \cup T} \, \sqrt{m} \, r_T.
\]
Thus, by Massart's lemma, since $|S \cup T| \leq 2m$, the following
inequality holds:
\begin{align*}
\h \R^\diamond_{S, T}(\cH) 
\leq r_T \, r_{S \cup T} \Lambda_1 \sqrt{\frac{2 \log (4m)}{m}} 
\leq r^2_{S \cup T} \Lambda_1 \sqrt{\frac{2 \log (4m)}{m}},
\end{align*}
which completes the proof.
\end{proof}

Notice that the bound on the Rademacher complexity in Lemma~\ref{lemma:rad}is non-trivial since it depends on the samples $S$ and $T$, while a standard Rademacher complexity for non-data-dependent hypothesis set containing $\sH_S$ would require taking a maximum over all samples $S$ of size $m$.\\

\begin{lemma}
\label{lemma:linear-rad}
Suppose $\sX = \Rset^N$, and for every sample $S \in \sZ^m$ we associate a matrix $A_S \in \Rset^{d \times N}$ for some $d > 0$, and let $\cW_{S, \Lambda} = \{w \in \Rset^d:\ \|A_S^\top w\|_2 \leq \Lambda\}$ for some $\Lambda > 0$. Consider the hypothesis set $\sH_S := \set[\bigg]{x \mapsto w^\top A_S x \colon \ w \in \cW_{S, \Lambda}}$. Then, the empirical Rademacher complexity of the family of data-dependent hypothesis sets $\cH = (\sH_S)_{S \in \sZ^m}$ 
can be upper-bounded as follows:
\[
\h \R^\diamond_{S, T}(\cH) \leq \frac{\Lambda \sqrt{\sum_{i = 1}^m \|x^T_i \|_2^2}}{m} \leq \frac{\Lambda r}{\sqrt{m}},
\]
where $r = \sup_{i \in [m]} \| x^T_i \|_2$.
\end{lemma}
\begin{proof}
Let $X_T = [x^T_1 \, \cdots \, x^T_m]$. The following inequalities hold:
\begin{align*}
\h \R^\diamond_{S, T}(\cH) 
= \frac{1}{m} \E_{\bsigma} \bigg[\sup_{h \in \sH_{S, T}^\bsigma} \sum_{i =
  1}^m \sigma_i h(x^T_i) \bigg]
& = \frac{1}{m} \E_{\bsigma} \bigg[\sup_{w:\ \| A_S^\top w \|_2
  \leq \Lambda} w^\top A_S X_T \bsigma \bigg]\\
& \leq \frac{\Lambda}{m} \E_{\bsigma} \big[\| X_T \bsigma \|_2 \big] &
                                                                     (\text{Cauchy-Schwarz})\\
& \leq \frac{\Lambda}{m} \sqrt{\E_{\bsigma} \big[\| X_T \bsigma \|^2_2
  \big]} & (\text{Jensen's ineq.})\\
& \leq \frac{\Lambda}{m} \sqrt{\E_{\bsigma} \bigg[\sum_{i, j = 1}^m
  \sigma_i \sigma_j (x^T_i \cdot x^T_j)
  \bigg]} \\
& = \frac{\Lambda \sqrt{\sum_{i = 1}^m \| x^T_i \|_2^2}}{m},
\end{align*}
which completes the proof.
\end{proof}

\subsection{Concentration}

\begin{lemma}
\label{lemma:concentration}
Let $\cH$ a family of $\beta$-stable data-dependent hypothesis sets.
Then, for any $\delta > 0$, with probability at least $1 - \delta$ (over the draw of two samples $S$ and $T$ with size $m$), the following inequality holds:
\[
\Big| \h \R^\diamond_{S, T}(\cG) - \R^\diamond_m(\cG) \Big| \leq
\sqrt{\frac{[(m \beta + 1)^2 + m^2 \beta^2] \log \frac{2}{\delta}}{2m}}.
\]
\end{lemma}
\begin{proof}
Let $T'$ be a sample differing from $T$ only by point.
Fix $\eta > 0$. For any $\bsigma$, by definition of the supremum,
there exists $h' \in \sH_{S, T'}^\bsigma$ such that:
\begin{align*}
\sum_{i = 1}^m \sigma_i L(h', z^T_i) 
\geq \sup_{h \in \sH_{S, T'}^\bsigma} \sum_{i = 1}^m \sigma_i L(h, z^{T'}_i) - \eta.
\end{align*}
By the $\beta$-stability of $\cH$, there exists $h \in \sH_{S,
  T}^\bsigma$ such that for any $z \in \sZ$,
$|L(h', z) - L(h, z)| \leq \beta$. Thus, we have
\begin{align*}
\sup_{h \in \sH_{S, T'}^\bsigma} \sum_{i = 1}^m \sigma_i L(h, z^{T'}_i) 
\leq \sum_{i = 1}^m \sigma_i L(h', z^{T'}_i) + \eta
\leq \sum_{i = 1}^m [\sigma_i (L(h, z^{T'}_i) + \beta)] + \eta.
\end{align*}
Since the inequality holds for all $\eta > 0$, we have
\[
\frac{1}{m} \sup_{h \in \sH_{S, T'}^\bsigma} \sum_{i = 1}^m \sigma_i L(h, z^{T'}_i)
\leq \frac{1}{m}\sum_{i = 1}^m \sigma_i (L(h, z^{T'}_i) + \beta)
\leq \frac{1}{m}\sup_{h \in \sH_{S, T}^\bsigma} \sum_{i = 1}^m \sigma_i L(h,
z^T_i) + \beta + \frac{1}{m}.
\]
Thus, replacing $T$ by $T'$ affects $\h \R^\diamond_{S, T}(\cG)$
by at most $\beta+\frac{1}{m}$. By the same argument, changing sample
$S$ by one point modifies $\h \R^\diamond_{S, T}(\cG)$ at
most by $\beta$. Thus, by McDiarmid's inequality,
for any $\delta > 0$, with probability at least $1 - \delta$,
the following inequality holds:
\[
\Big| \h \R^\diamond_{S, T}(\cG) - \R^\diamond_{m}(\cG) \Big| \leq
\sqrt{\frac{[(m \beta + 1)^2 + m^2 \beta^2] \log \frac{2}{\delta}}{2m}}.
\]
This completes the proof.
\end{proof}

\newpage
\section{Proof of Lemma~\ref{lem:cvstab-diam-bound}}
\label{app:cvstab-diam-bound}

\begin{proof}
  Let $S \in \sZ^m$, $z \in S$, and $z' \in \sZ$. For any
  $h \in \sH_S$ and $h' \in \sH_{S^{z \leftrightarrow z'}}$, by the
  $\beta$-uniform stability of $\cH$, there exists $h'' \in \sH_S$
  such that $L(h', z) - L(h'', z) \leq \beta$. Thus,
\[
L(h', z) - L(h, z) 
= L(h', z) - L(h'', z) + L(h'', z) - L(h, z) 
\leq \beta + \sup_{h'' \!, \, h \in \sH_S} L(h'', z) - L(h, z).
\] 
This implies the inequality
\[
\sup_{h \in \sH_S, h' \in \sH_{S^{z \leftrightarrow z'}}} L(h', z) -
L(h, z) \leq \beta + \sup_{h'' \!, \, h \in \sH_S} L(h'', z) - L(h, z),
\]
and the lemma follows.
\end{proof}

\newpage
\section{Proof of Theorem~\ref{th:hss2}}
\label{app:hss2}

In this section, we present the proof of Theorem~\ref{th:hss2}.

% \begin{theorem}{th:hss2}
% Let $\cH = (\sH_S)_{S \in \sZ^m}$ be a family of data-dependent
% hypothesis sets. Then, for any $\e > 0$ with $n \e^2 \geq 2$
% and any $n \geq 1$, the following inequality holds:
% \begin{equation*}
% \P\bigg[\sup_{h \in \sH_S} R(h) - \h R_S(h) > \e \bigg]
% \leq \exp\left[ - \frac{2}{\eta} \frac{mn}{m + n} \left[\frac{\e}{2} -
%     \max_{U \in \sZ^{m + n}} \h \R^\diamond_{U, m}(\cG) -
%     \sqrt{\frac{\log(2e) (m + n)^3}{2 (mn)^2}} \right]^2 \right],
% \end{equation*}
% where $\eta = \frac{m + n}{m + n - \frac{1}{2}} \frac{1}{1 - \frac{1}{2
%   \max\set{m, n}}} \approx 1$. For $m = n$, the inequality becomes:
% \begin{equation*}
% \P\bigg[\sup_{h \in \sH_S} R(h) - \h R_S(h) > \e \bigg]
% \leq \exp\left[ - \frac{m}{\eta} \left[\frac{\e}{2} -
%   \max_{U \in \sZ^{m + n}} \h \R^\diamond_{U, m}(\cG) - 2\sqrt{\frac{\log(2e)}{m}} \right]^2 \right].
% \end{equation*}
% \end{theorem}
\begin{proof}
  We will use the following symmetrization result, which holds for any
  $\e > 0$ with $n \e^2 \geq 2$ for data-dependent hypothesis sets
  (Lemma~\ref{lemma:symmetrization} below):
\begin{equation} \label{eq:symmetrization}
\P_{S \sim \sD^m} \bigg[ \sup_{h \in \sH_S} R(h) - \h R_S(h) > \e \bigg] 
\leq 2 \P_{\substack{S \sim \sD^m \\ T \sim \sD^n \mspace{6mu}}}
\bigg[ \sup_{h \in \sH_S} \h R_T(h) - \h R_S(h) > \frac{\e}{2}  \bigg] .
\end{equation}
Thus, we will seek to bound the right-hand side as follows, where we
write $(S, T) \sim U$ to indicate that the sample $S$ of size $m$ is 
drawn uniformly without replacement from $U$ and that $T$ is 
the remaining part of $U$, that is $(S, T) = U$:
\begin{align*}
& \P_{\substack{S \sim \sD^m \\ T \sim \sD^n \mspace{6mu}}} \bigg[
  \sup_{h \in \sH_S} \h R_T(h) - \h R_S(h) > \frac{\e}{2}  \bigg] \\
& = \E_{U \sim \sD^{m + n}} \Bigg[ \P_{\substack{(S, T) \sim U\\|S| = m, |T| =
  n}} \bigg[ \sup_{h \in \sH_S} \h R_T(h) - \h R_S(h) > \frac{\e}{2}
  \bigg] \ \bigg\mid \ U \Bigg]\\
& \leq \E_{U \sim \sD^{m + n}} \Bigg[ \P_{\substack{(S, T) \sim U\\|S| = m, |T| =
  n}} \bigg[ \sup_{h \in \ov \sH_{U, m}} \h R_T(h) - \h R_S(h) > \frac{\e}{2}
  \bigg] \ \bigg\mid \ U \Bigg].
\end{align*}
To upper bound the probability inside the expectation, we use an
extension of McDiarmid's inequality to sampling without replacement
\citep{CortesMohriPechyonyRastogi2008}, which applies to symmetric
functions. We can apply that extension to
$\Phi(S) = \sup_{h \in \ov \sH_{U, m}} \h R_T(h) - \h R_S(h)$, for a
fixed $U$, since $\Phi(S)$ is a symmetric function of the sample
points $z_1, \ldots, z_m$ in $S$.  Changing one point in $S$ affects
$\Phi(S)$ at most by $\frac{1}{m} + \frac{1}{m} = \frac{m + u}{mu}$,
thus, by the extension of McDiarmid's inequality to sampling without
replacement, for a fixed $U \in \sZ^{m + n}$, the following inequality
holds:
\begin{equation} \label{eq:trans}
\P_{\substack{(S, T) \sim U\\|S| = m, |T| =
  n}} \bigg[ \sup_{h \in \ov \sH_{U, m}} \h R_T(h) - \h R_S(h) > \frac{\e}{2}
  \bigg]
\leq \exp\bigg[ - \frac{2}{\eta} \frac{mn}{m + n} \bigg(\frac{\e}{2} -
  \E[\Phi(S)] \bigg)^2 \bigg],  
\end{equation}
where
$\eta = \frac{m + n}{m + n - \frac{1}{2}} \frac{1}{1 - \frac{1}{2
    \max\set{m, n}}} \leq 3$. Plugging in the bound on $\E[\Phi(S)]$ 
of Lemma~\ref{lemma:trans} below, and setting 
\[\e = \max_{U \in \sZ^{m + n}} 2\h \R^\diamond_{U, m}(\cG) + 3\sqrt{\left(\tfrac{1}{m} + \tfrac{1}{n}\right) \log(\tfrac{2}{\delta})} + 2\sqrt{\left(\tfrac{1}{m} + \tfrac{1}{n}\right)^3 mn},\] 
which satisfies $n\e^2 \geq 2$, it is easy to check that the RHS in \eqref{eq:trans} becomes smaller than $\frac{\delta}{2}$. This in turn implies, via \eqref{eq:symmetrization}, that the probability that $\sup_{h \in \sH_S} R(h) - \h R_S(h) > \e$ is at most $\delta$, completing the proof.
\end{proof}

% \section{Symmetrization lemma}
% \label{app:symmetrization}

% In this section, we 
The following lemma shows that the standard symmetrization lemma holds
for data-dependent hypothesis sets. This observation was already made
by \cite{Gat2001} (see also Lemma~2 in
\citep{CannonEttingerHushScovel2002}) for the symmetrization lemma of
\cite{Vapnik1998}[p. 139], used by the author in the case $n =
m$. However, that symmetrization lemma of \cite{Vapnik1998} holds only
for random variables taking values in $\set{0, 1}$ and its proof is
not complete since the hypergeometric inequality is not proven.

\begin{lemma}
\label{lemma:symmetrization}
Let $n \geq 1$ and fix $\e > 0$ such that $n \e^2 \geq 2$. Then, 
the following inequality holds:
\begin{equation*}
\P_{S \sim \sD^m} \bigg[ \sup_{h \in \sH_S} R(h) -
  \h R_S(h) > \e \bigg]
\leq 2 \P_{\substack{S \sim \sD^m \\ T \sim \sD^n \mspace{6mu}}}\bigg[\sup_{h \in \sH_S}  \h R_T(h_S) - \h R_S(h_S) > \frac{\e}{2} \bigg].
\end{equation*}
\end{lemma}

\begin{proof}
  The proof is standard. Below, we are giving a concise version mainly
  for the purpose of verifying that the data-dependency of the
  hypothesis set does not affect its correctness.

Fix $\eta > 0$. By definition of the supremum, there exists
$h_S \in \sH_S$ such that
\[
\sup_{h \in \sH_S} R(h) - \h R_S(h) - \eta \leq R(h_S) - \h R_S(h_S).
\]
Since $\h R_T(h_S) - \h R_S(h_S)  = \h R_T(h_S) - R(h_S) + R(h_S) - \h
R_S(h_S)$, we can write
\begin{align*}
1_{\h R_T(h_S) - \h R_S(h_S) > \frac{\e}{2}}
%= 1_{\h R_T(h_S) - R(h_S) + R(h_S) - \h R_S(h_S) > \frac{\e}{2}}
& \geq 1_{\h R_T(h_S) - R(h_S) > -\frac{\e}{2}} 1_{R(h_S) - \h R_S(h_S) > \e} 
= 1_{R(h_S) - \h R_T(h_S) < \frac{\e}{2}} 1_{R(h_S) - \h R_S(h_S) > \e}.
\end{align*}
Thus, for any $S \in \sZ^m$, taking the expectation of both sides with
respect to $T$ yields
\begin{align*}
\P_{T \sim \sD^n}\bigg[ \h R_T(h_S) - \h R_S(h_S) > \frac{\e}{2} \bigg]
& \geq \P_{T \sim \sD^n} \bigg[ R(h_S) - \h R_T(h_S) < \frac{\e}{2}
  \bigg] \ 1_{R(h_S) - \h R_S(h_S) > \e}\\
& = \bigg[ 1 - \P_{T \sim \sD^n} \bigg[ R(h_S) - \h R_T(h_S) \geq \frac{\e}{2}
  \bigg] \bigg] \ 1_{R(h_S) - \h R_S(h_S) > \e}\\
& \leq \bigg[ 1 - \frac{4 \Var[L(h_S, z)]}{n \e^2} \bigg] \ 1_{R(h_S) -
  \h R_S(h_S) > \e} & \text{(Chebyshev's ineq.)}\\
& \geq \bigg[ 1 - \frac{1}{n \e^2} \bigg] \ 1_{R(h_S) -
  \h R_S(h_S) > \e} \ ,
\end{align*}
where the last inequality holds since $L(h_S, z)$ takes values in $[0,
1]$:
\begin{align*}
\Var[L(h_S, z)] 
= \E_{z \sim \sD}[L^2(h_S, z)] - \E_{z \sim \sD}[L(h_S, z)]^2
& \leq \E_{z \sim \sD}[L(h_S, z)] - \E_{z \sim \sD}[L(h_S, z)]^2\\
& = \E_{z \sim \sD}[L(h_S, z)] (1 - \E_{z \sim \sD}[L(h_S, z)]) \leq \frac{1}{4}.
\end{align*}
Taking expectation with respect to $S$ gives
\begin{align*}
\P_{\substack{S \sim \sD^m \\ T \sim \sD^n \mspace{6mu}}}\bigg[ \h R_T(h_S) - \h R_S(h_S) > \frac{\e}{2} \bigg]
& \geq \bigg[ 1 - \frac{1}{n \e^2} \bigg] \P_{S \sim \sD^m} \bigg[ R(h_S) -
  \h R_S(h_S) > \e \bigg]\\
& \geq \frac{1}{2} \P_{S \sim \sD^m} \bigg[ R(h_S) -
  \h R_S(h_S) > \e \bigg] & \text{$(n \e^2 \geq 2)$}\\
& \geq \frac{1}{2} \P_{S \sim \sD^m} \bigg[ \sup_{h \in \sH_S} R(h) -
  \h R_S(h) > \e + \eta \bigg].
\end{align*}
Since the inequality holds for all $\eta > 0$, by the right-continuity
of the cumulative distribution function, it implies
\[
\P_{\substack{S \sim \sD^m \\ T \sim \sD^n \mspace{6mu}}}\bigg[ \h R_T(h_S) - \h R_S(h_S) > \frac{\e}{2} \bigg]
\geq \frac{1}{2} \P_{S \sim \sD^m} \bigg[ \sup_{h \in \sH_S} R(h) -
  \h R_S(h) > \e \bigg].
\]
Since $h_S$ is in $\sH_S$, by definition of the supremum, we have 
\[
\P_{\substack{S \sim \sD^m \\ T \sim \sD^n \mspace{6mu}}}\bigg[\sup_{h \in \sH_S} \h R_T(h) - \h R_S(h) > \frac{\e}{2} \bigg]
\geq \P_{\substack{S \sim \sD^m \\ T \sim \sD^n \mspace{6mu}}}\bigg[
\h R_T(h_S) - \h R_S(h_S) > \frac{\e}{2} \bigg],
\]
which completes the proof.
\end{proof}

% \section{Transductive Rademacher complexity bound}
% \label{app:trans}

The following lemma provides a bound on $\E[\Phi(S)]$:

\begin{lemma}
\label{lemma:trans}
Fix $U \in \sZ^{m+ n}$. Then, the following upper bound holds:
\begin{equation*}
\E_{\substack{(S, T) \sim U\\|S| = m, |T| = n}} \bigg[\sup_{h \in \ov
  \sH_{U, m}} \h R_T(h) - \h R_S(h) \bigg]
\leq \h \R^\diamond_{U, m}(\cG) + \sqrt{\frac{\log(2e) (m + n)^3}{2 (mn)^2}}.
\end{equation*}
For $m = n$, the inequality becomes:
\begin{equation*}
\E_{\substack{(S, T) \sim U\\|S| = m, |T| = n}} \bigg[\sup_{h \in \ov
  \sH_{U, m}} \h R_T(h) - \h R_S(h) \bigg]
\leq \h \R^\diamond_{U, m}(\cG) + 2 \sqrt{\frac{\log(2e)}{m}}.
\end{equation*}
\end{lemma}
\begin{proof}
  The proof is an extension of the analysis of \emph{maximum
    discrepancy} in \citep{BartlettMendelson2002}. Let $|\bsigma|$
  denote $\sum_{i = 1}^{m + n} \sigma_i$ and let
  $I \subseteq \Big[ \mspace{-6mu} - \frac{(m + n)^2}{m}, \frac{(m +
    n)^2}{n} \Big]$ denote the set of values $|\bsigma|$ can take. For
  any $q \in I$, define $s(q)$ as follows:
\[
s(q) = \E_\bsigma \bigg[ \sup_{h \in \ov \sH_{U, m}} \frac{1}{m + n}
\sum_{i = 1}^{m + n} \sigma_i L(h, z^U_i)  \bigg| \, |\bsigma| = q \bigg].
\]
Let $|\bsigma|_+$ denote the number of positive $\sigma_i$s, taking
value $\frac{m + n}{n}$, then $|\bsigma|$ can be expressed as follows:
\begin{equation}
\label{eq:sum}
|\bsigma|
= \sum_{i = 1}^{m + n} \sigma_i 
= |\bsigma|_+ \frac{m + n}{n} - (m + n - |\bsigma|_+) \frac{m + n}{m}
= \frac{(m + n)^2}{mn} (|\bsigma|_+ - n).
\end{equation}
Thus, we have $|\bsigma| = 0$ iff $|\bsigma|_+ = m$, and the condition
($|\bsigma| = 0$) precisely corresponds to having the equality
\[
\frac{1}{m + n} \sum_{i = 1}^{m + n} \sigma_i L(h, z^U_i) =
\h R_T(h) - \h R_S(h),
\]
where $S$ is the sample of size $m$ defined by those $z_i$s for which
$\sigma_i$ takes value $\frac{m + n}{n}$. In view of that, we
have
\[
\E_{\substack{(S, T) \sim U\\|S| = m, |T| = n}} \bigg[\sup_{h \in \ov
  \sH_{U, m}} \h R_T(h) - \h R_S(h) \bigg] = s(0).
\]
Let $q_1, q_2 \in I$, with 
$q_1 = p_1 \frac{m + n}{n} - (m + n - p_1) \frac{m + n}{m}$,
$q_2 = p_2 \frac{m + n}{n} - (m + n - p_2) \frac{m + n}{m}$
and $q_1 \leq q_2$. Then, we can write
\begin{align*}
s(q_1) 
& = \E \left[ \sup_{g \in G} 
\sum_{i = 1}^{p_1} \frac{1}{n} L(h, z_i) 
- \sum_{i = p_1 + 1}^{m + n} \frac{1}{m} L(h, z_i) 
 \right]\\
s(q_2) 
& = \E \Bigg[ \sup_{g \in G} 
\sum_{i = 1}^{p_1} \frac{1}{n} L(h, z_i) 
- \sum_{i = p_1 + 1}^{m + n} \frac{1}{m} L(h, z_i) 
+ \sum_{i = p_1+ 1}^{p_2} \left[ \frac{1}{n} + \frac{1}{m} \right] L(h, z_i)
 \Bigg].
\end{align*}
Thus, we have the following Lipschitz property:
\begin{align*}
|s(q_2) - s(q_1)| 
\leq |p_2 - p_1| \bigg[\frac{1}{m} + \frac{1}{n} \bigg]
& = |(p_2 - n) - (p_1 - n)| \bigg[\frac{1}{m} + \frac{1}{n} \bigg] &
                                                                     \text{(using \eqref{eq:sum})}\\
& = |q_2 - q_1| \frac{mn}{(m + n)^2} \bigg[\frac{1}{m} + \frac{1}{n} \bigg]\\
& = \frac{|q_2 - q_1|}{m + n}.
\end{align*}
By this Lipschitz property, we can write
\[
\P\Big[ \big| s(|\bsigma|) - s(\E[|\bsigma|]) \big| > \e \Big]
\leq \P\Big[ \big| |\bsigma| - \E[|\bsigma|] \big| > (m + n) \e \Big]
\leq 2 \exp \bigg[ -2 \frac{(mn)^2 \e^2}{(m + n)^3} \bigg],
\]
since the range of each $\sigma_i$ is
$\frac{m + n}{n} + \frac{m + n}{m} = \frac{(m + n)^2}{mn}$.
We now use this inequality to bound the second moment of
$Z = s(|\bsigma|) - s(\E[|\bsigma|]) = s(|\bsigma|) -
s(0)$, as follows, for any $u \geq 0$:
\begin{align*}
\E[Z^2] 
& = \int_{0}^{+\infty} \P[Z^2 > t] \, dt\\
& = \int_{0}^{u} \P[Z^2 > t] \, dt + \int_{u}^{+\infty} \P[Z^2 > t] \, dt\\
& \leq u + 2 \int_{u}^{+\infty} \exp \bigg[ -2 \frac{(mn)^2 t}{(m +
  n)^3} \bigg] \, dt\\
& \leq u + \bigg[ \frac{(m +
  n)^3}{(mn)^2} \exp \bigg[ -2 \frac{(mn)^2 t}{(m +
  n)^3} \bigg] \bigg]_u^{+\infty} \\
& = u + \frac{(m +
  n)^3}{(mn)^2} \exp \bigg[ -2 \frac{(mn)^2 u}{(m +
  n)^3} \bigg].
\end{align*}
Choosing $u = \frac{1}{2} \frac{\log(2) (m + n)^3}{(mn)^2}$ to
minimize the right-hand side gives
$\E[Z^2] \leq \frac{\log(2e) (m + n)^3}{2 (mn)^2}$. By Jensen's
inequality, this implies
$\E[|Z|] \leq \sqrt{\frac{\log(2e) (m + n)^3}{2 (mn)^2}}$ and
therefore
\[
\E_{\substack{(S, T) \sim U\\|S| = m, |T| = n}} \bigg[\sup_{h \in \ov
  \sH_{U, m}} \h R_T(h) - \h R_S(h) \bigg] = s(0)
\leq \E[s(|\bsigma|)] + \sqrt{\frac{\log(2e) (m + n)^3}{2 (mn)^2}}.
\]
Since we have $\E[s(|\bsigma|)] = \h \R_{U, m}^\diamond(\cG)$, this
completes the proof.
\end{proof}

\newpage
\section{Proof of Theorem~\ref{th:hss}}
\label{app:hss}

In this section, we present the full proof of Theorem~\ref{th:hss}. The proof of each of the three bounds \eqref{eq:rad-hss}, \eqref{eq:dp-hss} and \eqref{eq:induction-hss} are given in separate subsections.
% \begin{theorem}{th:hss}
% Let $\cH = (\sH_S)_{S \in \sZ^m}$ be a $\beta$-stable family of data-dependent
% hypothesis sets. Then, for any $\delta > 0$, with probability at least
% $1 - \delta$ over the draw of a sample $S \sim \sZ^m$, the following
% inequality holds for all $h \in H_S$:
% \begin{equation}
% \forall h \in \sH_S, 
% R(h) \leq \h R_S(h) + \min\set{2 \R^\diamond_m(\cG), \avgcvstab} + [1 + 2 \beta m] \sqrt{\frac{\log \frac{1}{\delta}}{2m}}.
% \end{equation}
% \end{theorem}

\subsection{Proof of bound~\eqref{eq:rad-hss}}
\begin{proof}
  For any two samples $S, S'$, define the $\Psi(S, S')$ as follows:
\begin{equation*}
\Psi(S, S') = \sup_{h \in \sH_{S}} R(h) - \h R_{S'}(h) .
\end{equation*}
The proof consists of applying McDiarmid's inequality to $\Psi(S, S)$.
For any sample $S'$ differing from $S$ by one point,
we can decompose $\Psi(S, S) - \Psi(S', S')$ as follows:
\begin{equation*}
\Psi(S, S) - \Psi(S', S') 
= \big[ \Psi(S, S) - \Psi(S, S') \big] + \big[ \Psi(S, S') - \Psi(S',
S') \big].
\end{equation*}
Now, by the sub-additivity of the $\sup$ operation, 
the first term can be upper-bounded as follows:
\begin{align*}
\Psi(S, S) - \Psi(S, S') 
& \leq \sup_{h \in \sH_{S}} \big[ R(h) - \h R_{S}(h) \big] - \big[ R(h) - \h
  R_{S'}(h) \big]\\
& \leq \sup_{h \in \sH_{S}} \frac{1}{m} \big[ L(h, z) - L(h, z') \big]
\leq \frac{1}{m},
\end{align*}
where we denoted by $z$ and $z'$ the labeled points differing in $S$
and $S'$ and used the $1$-boundedness of the loss function.

We now analyze the second term:
\begin{align*}
\Psi(S, S') - \Psi(S', S')
& = \sup_{h \in \sH_{S}} \big[ R(h) - \h R_{S'}(h) \big] - \sup_{h \in
  \sH_{S'}} \big[ R(h) - \h R_{S'}(h) \big] .
\end{align*}
By definition of the supremum, for any $\e > 0$, 
there exists $h \in \sH_S$ such that
\begin{equation*}
\sup_{h \in \sH_{S}} \big[ R(h) - \h R_{S'}(h) \big] - \e
\leq \big[ R(h) - \h R_{S'}(h) \big]
\end{equation*}
By the $\beta$-stability of $(H_S)_{S \in \sZ^m}$, there exists
$h' \in \sH_{S'}$ such that for all $z$,
$|L(h, z) - L(h', z)| \leq \beta$. In view of these
inequalities, we can write
\begin{align*}
\Psi(S, S') - \Psi(S', S')
& \leq \big[ R(h) - \h R_{S'}(h) \big] + \e - \sup_{h \in
  \sH_{S'}} \big[ R(h) - \h R_{S'}(h) \big] \\
& \leq \big[ R(h) - \h R_{S'}(h) \big] + \e - \big[ R(h') - \h
  R_{S'}(h') \big] \\
& \leq \big[ R(h) - R(h') \big] + \e + \big[ \h R_{S'}(h') - \h R_{S'}(h) \big] \\
& \leq \e + 2 \beta.
\end{align*}
Since the inequality holds for any $\e > 0$, it implies
that $\Psi(S, S') - \Psi(S', S') \leq 2 \beta$. Summing up 
the bounds on the two terms shows the following:
\begin{equation*}
\Psi(S, S) - \Psi(S', S') \leq \frac{1}{m} + 2\beta.
\end{equation*}
Thus, by McDiarmid's inequality, for any $\delta > 0$,
with probability at least $1 - \delta$, we have
\begin{equation}
\label{eq:conc}
\Psi(S, S) \leq \E[\Psi(S, S)] + (1 + 2 \beta m) \sqrt{\tfrac{1}{2m}\log(\tfrac{1}{\delta})}.
\end{equation}
We now bound $\E[\Psi(S, S)]$ by $2 \R^\diamond_m(\cG)$ as follows:
% completing the proof. 
% % seek a more explicit upper bound for the expectation appearing
% % on the right-hand side, in terms of the Rademacher complexity.
% The following sequence of inequalities holds:
\begin{align*}
& \E_{S \sim \sD^m}[\Psi(S, S)]\\
& = \E_{S \sim \sD^m} \bigg[ \sup_{h \in \sH_{S}} \big[ R(h) - \h
  R_{S}(h) \big] \bigg]\\
& = \E_{S \sim \sD^m} \bigg[ \sup_{h \in \sH_{S}} \Big[\E_{T \sim \sD^m}
  \big[ \h
  R_{T}(h) \big] - \h R_{S}(h) \Big] \bigg] & \text{(def. of $R(h)$)}\\
& \leq \E_{S, T \sim \sD^m} \bigg[ \sup_{h \in \sH_{S}} \h R_{T}(h) - \h
  R_{S}(h) \bigg] & \text{(sub-additivity of $\sup$)}\\
& = \E_{S, T \sim \sD^m} \bigg[ \sup_{h \in \sH_S}
  \frac{1}{m} \sum_{i = 1}^m \Big[ L(h, z^T_i) - L(h, z^S_i) \Big] \bigg]\\
& = \E_{S, T \sim \sD^m} \bigg[ \E_\bsigma \bigg[ \sup_{h \in \sH^\bsigma_{S, T}}
  \frac{1}{m} \sum_{i = 1}^m \sigma_i \Big[ L(h, z^T_i) - L(h, z^S_i) \Big]
  \bigg] \bigg] & \text{(symmetry)}\\
& \leq \E_{\substack{S, T \sim \sD^m\\\bsigma}} \bigg[ \sup_{h \in \sH^\bsigma_{S, T}}
  \frac{1}{m} \sum_{i = 1}^m \sigma_i L(h, z^T_i) + \sup_{h \in \sH^\bsigma_{S, T}}
  \frac{1}{m} \sum_{i = 1}^m -\sigma_i L(h, z^S_i) \bigg] & \text{(sub-additivity of $\sup$)}\\
& = \E_{\substack{S, T \sim \sD^m\\\bsigma}} \bigg[ \sup_{h \in \sH^\bsigma_{S, T}}
  \frac{1}{m} \sum_{i = 1}^m \sigma_i L(h, z^T_i) + \sup_{h \in \sH^{-\bsigma}_{T, S}}
  \frac{1}{m} \sum_{i = 1}^m -\sigma_i L(h, z^S_i) \bigg] & (\sH^\bsigma_{S, T} = \sH^{-\bsigma}_{T, S})\\
  & = \E_{\substack{S, T \sim \sD^m\\\bsigma}} \bigg[ \sup_{h \in \sH^\bsigma_{S, T}}
  \frac{1}{m} \sum_{i = 1}^m \sigma_i L(h, z^T_i) + \sup_{h \in \sH^{\bsigma}_{T, S}}
  \frac{1}{m} \sum_{i = 1}^m \sigma_i L(h, z^S_i) \bigg] & \text{(symmetry)}\\
& = 2 \R^\diamond_m(\cG). & \text{(linearity of expectation)}
\end{align*} 

% \ignore{
Now, we show that $\E_{S \sim \sD^m}[\Psi(S, S)] \leq \avgcvstab$. To do so,
first fix $\e > 0$. By definition of the supremum, for any
$S \in \sZ^m$, there exists $h_S$ such that the following inequality
holds:
\begin{equation*}
\sup_{h \in \sH_{S}} \big[ R(h) - \h R_{S}(h) \big] - \e \leq R(h_S) - \h
R_{S}(h_S).
\end{equation*}
Now, by definition of $R(h_S)$, we can write
\begin{equation*}
\E_{S \sim \sD^m}\big[ R(h_S) \big]
= \E_{S \sim \sD^m}\bigg[\E_{z \sim \sD}(L(h_S, z) \bigg]
= \E_{\substack{S \sim \sD^m\\z \sim \sD\mspace{14mu}}}\big[ L(h_S, z) \big].
\end{equation*}
% Let $S^{i \gets z}$ denote the result of modifying sample
% $S \in \sZ^m$ by replacing its $i$th element with $z$.  
Then, by the linearity of expectation, we can also write
\begin{equation*}
\E_{S \sim \sD^m}\big[ \h R_S(h_S) \big]
%= \frac{1}{m} \sum_{z \in S} \E_{S \sim \sD^m}\big[ L(h_S, z) \big]
= \E_{\substack{S \sim \sD^m\\z \sim S}}\big[ L(h_S, z) \big]
= \E_{\substack{S \sim \sD^m\\z' \sim \sD\mspace{14mu}\\z \sim
    S}}\big[ L(h_{S^{z \leftrightarrow z'}}, z') \big].
\end{equation*}
In view of these two equalities, we can now rewrite the upper bound as
follows:
\begin{align*}
\E_{S \sim \sD^m}\big[ \Psi(S, S) \big]
& \leq \E_{S \sim \sD^m}\big[ R(h_S) - \h R_{S}(h_S) \big] + \e\\
& = \E_{\substack{S \sim \sD^m\\z' \sim \sD\mspace{14mu}}}\big[ L(h_S, z') \big] -
  \E_{\substack{S \sim \sD^m\\z' \sim \sD\mspace{14mu}\\z \sim
    S}}\big[ L(h_{S^{z \leftrightarrow z'}}, z') \big] +
  \e\\
& = \E_{\substack{S \sim \sD^m\\z' \sim \sD\mspace{14mu}\\z \sim
    S}}\big[ L(h_S, z') - L(h_{S^{z \leftrightarrow z'}}, z') \big] +
  \e\\
& = \E_{\substack{S \sim \sD^m\\z' \sim \sD\mspace{14mu}\\z \sim
    S}}\big[L(h_{S^{z \leftrightarrow z'}}, z) - L(h_S, z)\big] +
  \e\\
&\leq \avgcvstab + \e.
\end{align*}
Since the inequality holds for all $\e > 0$, it implies
$\E_{S \sim \sD^m}\big[ \Psi(S, S) \big] \leq \avgcvstab$.  
Plugging in
these upper bounds on the expectation in the inequality
\eqref{eq:conc} completes the proof.
% } % ends ignore
\end{proof}

\subsection{Proof of bound~\eqref{eq:dp-hss}}
\label{sec:dp-hss}

The proof of bound~\eqref{eq:dp-hss} relies on recent techniques introduced in the differential privacy literature to derive improved generalization
guarantees for stable data-dependent hypothesis sets \citep{SteinkeUllman2017,BassilyNissimSmithSteinkeStemmerUllman2016} (see also \citep{McSherryTalwar2007}). Our proof also benefits from the recent improved stability results of \cite{FeldmanVondrak2018}. We will make use of the following lemma due to \citet[Lemma 1.2]{SteinkeUllman2017}, which reduces the task of deriving
a concentration inequality to that of upper bounding an
expectation of a maximum.

\begin{lemma}
\label{lemma:SU}
Fix $p \geq 1$. Let $X$ be a random variable with probability
distribution $\sD$ and $X_1, \ldots, X_p$ independent copies of
$X$. Then, the following inequality holds:
\[
\Pr_{X \sim \sD} \bigg[X \geq 2 \E_{X_k \sim \sD} \Big[\max \set[\big]{0, X_1,
\ldots, X_p} \Big] \bigg] \leq \frac{\log 2}{p}.
\]
\end{lemma}

We will also use the following result which, under a sensitivity
assumption, further reduces the task of upper bounding the expectation
of the maximum to that of bounding a more favorable expression.

\begin{lemma}[\citep{McSherryTalwar2007,BassilyNissimSmithSteinkeStemmerUllman2016,FeldmanVondrak2018}]
\label{lemma:diff-stab}
Let $f_1, \ldots, f_p\colon \sZ^m \to \Rset$ be $p$ functions
with sensitivity $\supdiam$. Let $\cA$ be the algorithm that, given a
dataset $S \in \sZ^m$ and a parameter $\e > 0$, returns the index
$k \in [p]$ with probability proportional to
$e^{\frac{\e f_k(S) }{2 \supdiam}}$. Then, $\cA$ is $\e$-differentially
private and, for any $S \in \sZ^m$, the following inequality holds:
\[
\max_{k \in [p]} \set[\big]{f_k(S)}
\leq \E_{\substack{k = \cA(\sfS)}}  \big[ f_k(S) \big] + \frac{2 \supdiam}{\e}
\log p.
\]
\end{lemma}
Notice that, if we define $f_{p + 1} = 0$, then, by the same result,
the algorithm $\cA$ returning the index $k \in [p + 1]$ with
probability proportional to
$e^{\frac{\e f_k(S) 1_{k \neq (p + 1)} }{2 \supdiam}}$ is
$\e$-differentially private and the following inequality holds for any
$S \in \sZ^m$:
\begin{equation}
\label{eq:diff}
\max\set[\Big]{0, \max_{k \in [p]} \set[\big]{f_k(S)}}
= \max_{k \in [p + 1]} \set[\big]{f_k(S)}
\leq \E_{\substack{k = \cA(\sfS)}}  \big[ f_k(S) \big] + \frac{2 \supdiam}{\e}
\log (p + 1).
\end{equation}

% In this section, we present the proof of Theorem~\ref{th:hss-diff}.
% \begin{theorem}{th:hss-diff}
% Let $\cH = (\sH_S)_{S \in \sZ^m}$ be a $\beta$-stable family of
% data-dependent hypothesis sets. Then, for any $\delta \in (0, 1)$,
% with probability at least $1 - \delta$ over the draw of a sample
% $S \sim \sZ^m$, the following inequality holds for all $h \in \sH_S$:
% \begin{equation*}
% R(h) \leq \h R_S(h) + 
% \min \set[\Bigg]{2 \R^\diamond_m(\cG) + [1 + 2 \beta m] \sqrt{\frac{\log
%     \frac{2}{\delta}}{2m}},
% \sqrt{e} \supcvstab +
% 4 \sqrt{\left[\tfrac{1}{m} + 2 \beta \right] \log \left[
%     \tfrac{6}{\delta} \right]}}.
% \end{equation*}
% \end{theorem}

Equipped with these lemmas, we can now turn to the proof of bound~\eqref{eq:dp-hss}:
\begin{proof}
  For any two samples $S, S'$ of size $m$, define $\Psi(S, S')$ as follows:
\begin{equation*}
  \Psi(S, S') = \sup_{h \in \sH_{S}} R(h) - \h R_{S'}(h) .
\end{equation*}
The proof consists of deriving a high-probability bound for
$\Psi(S, S)$. To do so, by Lemma~\ref{lemma:SU} applied to the random
variable $X = \Psi(S, S)$, it suffices to bound
$\E_{\substack{\sfS \sim \sD^{pm}}} \Big[ \max\set[\Big]{0, \max_{k
    \in [p]} \set[\big]{\Psi(S_k, S_k)}} \Big]$, where
$\sfS = (S_1, \ldots, S_p)$ with $S_k$, $k \in [p]$, independent
samples of size $m$ drawn from $\sD^m$.

To bound that expectation, we can use Lemma~\ref{lemma:diff-stab} and
instead bound
$\E_{\substack{\sfS \sim \sD^{pm}\\ k = \cA(\sfS)}} [\Psi(S_k,
S_k)]$, where $\cA$ is an $\e$-differentially private algorithm.

Now, to apply Lemma~\ref{lemma:diff-stab}, we first show that, for any
$k \in [p]$, the function $f_k \colon \sfS \to \Psi(S_k, S_k)$ is
$\supdiam$-sensitive with $\supdiam = \frac{1}{m} + 2\beta$. Fix
$k \in [p]$.  Let $\sfS' = (S'_1, \ldots, S'_p)$ be in $\sZ^{pm}$ and
assume that $\sfS'$ differs from $\sfS$ by one point. If they differ
by a point not in $S_k$ (or $S'_k$), then $f_k(\sfS) = f_k(\sfS')$. Otherwise, they differ only by a point in $S_k$ (or $S'_k$) and
$f_k(\sfS) - f_k(\sfS') = \Psi(S_k, S_k) - \Psi(S'_k, S'_k)$.  We
can decompose this term as follows:
\begin{equation*}
\Psi(S_k, S_k) - \Psi(S'_k, S'_k) 
= \big[ \Psi(S_k, S_k) - \Psi(S_k, S'_k) \big] + \big[ \Psi(S_k, S'_k) - \Psi(S'_k, S'_k) \big].
\end{equation*}
Now, by the sub-additivity of the $\sup$ operation, 
the first term can be upper-bounded as follows:
\begin{align*}
\Psi(S_k, S_k) - \Psi(S_k, S'_k) 
& \leq \sup_{h \in \sH_{S_k}} \big[ R(h) - \h R_{S_k}(h) \big] - \big[
  R(h) - \h
  R_{S'_k}(h) \big]\\
& \leq \sup_{h \in \sH_{S_k}} \frac{1}{m} \big[ L(h, z) - L(h, z') \big]
\leq \frac{1}{m},
\end{align*}
where we denoted by $z$ and $z'$ the labeled points differing in $S_k$
and $S'_k$ and used the $1$-boundedness of the loss function.

We now analyze the second term:
\begin{align*}
\Psi(S_k, S'_k) - \Psi(S'_k, S'_k)
& = \sup_{h \in \sH_{S_k}} \big[ R(h) - \h R_{S'_k}(h) \big] - \sup_{h \in
  \sH_{S'_k}} \big[ R(h) - \h R_{S'_k}(h) \big] .
\end{align*}
By definition of the supremum, for any $\eta > 0$, 
there exists $h \in \sH_{S_k}$ such that
\begin{equation*}
\sup_{h \in \sH_{S_k}} \big[ R(h) - \h R_{S'_k}(h) \big] - \eta
\leq \big[ R(h) - \h R_{S'_k}(h) \big]
\end{equation*}
By the $\beta$-stability of $(\sH_S)_{S \in \sZ^m}$, there exists
$h' \in \sH_{S'_k}$ such that for all $z$,
$\big| L(h, z) - L(h', z) \big| \leq \beta$. In view of these
inequalities, we can write
\begin{align*}
\Psi(S_k, S'_k) - \Psi(S'_k, S'_k)
& \leq \big[ R(h) - \h R_{S'_k}(h) \big] + \eta - \sup_{h \in
  \sH_{S'_k}} \big[ R(h) - \h R_{S'_k}(h) \big] \\
& \leq \big[ R(h) - \h R_{S'_k}(h) \big] + \eta - \big[ R(h') - \h R_{S'_k}(h') \big] \\
& \leq \big[ R(h) - R(h') \big] + \eta + \big[\h R_{S'_k}(h') - \h R_{S'_k}(h) \big] \\
& \leq \eta + 2 \beta.
\end{align*}
Since the inequality holds for any $\eta > 0$, it implies
that $\Psi(S_k, S'_k) - \Psi(S'_k, S'_k) \leq 2 \beta$. Summing up 
the bounds on the two terms shows the following:
\begin{equation*}
\Psi(S_k, S_k) - \Psi(S'_k, S'_k) \leq \frac{1}{m} + 2\beta.
\end{equation*}

Having established the $\supdiam$-sensitivity of the functions $f_k$,
$k \in [p]$, we can now apply Lemma~\ref{lemma:diff-stab}. Fix
$\e > 0$.  Then, by Lemma~\ref{lemma:diff-stab} and \eqref{eq:diff},
the algorithm $\cA$ returning $k \in [p + 1]$ with probability
proportional to
$e^{\frac{\e \Psi(S_k, S_k) 1_{k \neq (p + 1)} }{2 \supdiam}}$ is
$\e$-differentially private and, for any sample $\sfS \in \sZ^{pm}$,
the following inequality holds:
\[
\max\set[\Big]{0, \max_{k \in [p]} \set[\big]{\Psi(S_k, S_k)}}
\leq \E_{\substack{k = \cA(\sfS)}}  \big[ \Psi(S_k, S_k) \big] + \frac{2 \supdiam}{\e}
\log (p + 1).
\]
Taking the expectation of both sides yields
\begin{equation}
\label{eq:ExpectationOfMaxBound}
\E_{\substack{\sfS \sim \sD^{pm}}} \Big[ \max\set[\Big]{0, \max_{k \in [p]}
  \set[\big]{\Psi(S_k, S_k)}} 
\Big] 
\leq \E_{\substack{\sfS \sim \sD^{pm}\\ k = \cA(\sfS)}}  \big[ \Psi(S_k, S_k) \big] + \frac{2 \supdiam}{\e}
\log (p + 1).
\end{equation}
We will show the following upper bound on the expectation:
$\E_{\substack{\sfS \sim \sD^{pm}\\ k = \cA(\sfS)}} \big[ \Psi(S_k, S_k) \big]
\leq (e^\e - 1) + e^\e\supcvstab$.  To do so, first fix $\eta > 0$. By
definition of the supremum, for any $S \in \sZ^m$, there exists
$h_S \in \sH_S$ such that the following inequality holds:
\begin{equation*}
\sup_{h \in \sH_{S}} \big[ R(h) - \h R_{S}(h) \big] - \eta \leq R(h_S) - \h
R_{S}(h_S).
\end{equation*}
In what follows, we denote by $\sfS^{k, z \leftrightarrow z'} \in \sZ^{pm}$ the
result of modifying $\sfS = (S_1, \ldots, S_p) \in \sZ^{pm}$ by
replacing $z \in S_k$ with $z'$. 
% Similarly,
% $S_k^{i \gets z}$ denotes the result of modifying $S_k \in \sZ^m$ by
% replacing its $i$th element with $z$.

Now, by definition of the algorithm $\cA$, we can write:
\begin{align*}
\E_{\substack{\sfS \sim \sD^{pm}\\ k = \cA(\sfS)}} \big[ R(h_{S_k}) \big]
& = \E_{\substack{\sfS \sim \sD^{pm}\\ k = \cA(\sfS)}} \bigg[\E_{z' \sim
  \sD} [L(h_{S_k}, z')] \bigg] & \text{(def. of $R(h_{S_k})$)}\\
& = \E_{\substack{\sfS \sim \sD^{pm}\\ z' \sim
  \sD}}  \bigg[ \sum_{k = 1}^p \Pr[\cA(\sfS) = k] \, L(h_{S_k}, z') \bigg] & \text{(def. of $\E_{k = \cA(\sfS)}$)}\\
& = \sum_{k = 1}^p \, \E_{\substack{\sfS \sim \sD^{pm}\\ z' \sim
  \sD}}  \bigg[\Pr[\cA(\sfS) = k] \, L(h_{S_k}, z') \bigg] & \text{(linearity
                                                          of expect.)}\\
& \leq \sum_{k = 1}^p \, \E_{\substack{\sfS \sim \sD^{pm}\\ z' \sim\sD,\ z \sim S_k}}  \bigg[e^\e \Pr[\cA(\sfS^{k, z \leftrightarrow z'}) = k] \,
  L(h_{S_k}, z') \bigg] & \text{($\e$-diff. privacy of $\cA$)}\\
& = \sum_{k = 1}^p \, \E_{\substack{\sfS \sim \sD^{pm}\\ z' \sim\sD,\ z \sim S_k}}  \bigg[e^\e \Pr[\cA(\sfS) = k] \,
  L(h_{S_k^{z \leftrightarrow z'}}, z) \bigg] & \text{(swapping $z'$
                                                 and $z$)} \\
&\leq \sum_{k = 1}^p \, \E_{\substack{\sfS \sim \sD^{pm}\\ z' \sim\sD,\ z \sim S_k}}  \bigg[e^\e \Pr[\cA(\sfS) = k] \,
  L(h_{S_k}, z) \bigg] + e^\e \supcvstab. & \text{(By Lemma~\ref{lem:supcvstab-bound} below)}
\end{align*}

Now, observe that $\E_{z \sim S_k}[L(h_{S_k}, z)]$
coincides with $\h R(h_{S_k})$, the empirical loss of $h_{S_k}$.
Thus, we can write
\begin{equation*}
\E_{\substack{\sfS \sim \sD^{pm}\\ k = \cA(\sfS)}} \big[ R(h_{S_k}) \big]
\leq \sum_{k = 1}^p \, \E_{\substack{\sfS \sim \sD^{pm}\\ z \sim S_k}}  \bigg[e^\e \Pr[\cA(\sfS) = k] \,
  \h R_{S_k}(h_{S_k}) \bigg] + e^\e \supcvstab,
\end{equation*}
and therefore
\begin{align*}
\E_{\substack{\sfS \sim \sD^{pm}\\ k = \cA(\sfS)}} \big[ \Psi(S_k, S_k) \big]
& \leq \sum_{k = 1}^p \, \E_{\substack{\sfS \sim \sD^{pm}\\ k = \cA(\sfS)}}  \bigg[(e^\e - 1) \h R_{S_k}(h_{S_k}) \bigg] + e^\e \supcvstab + \eta\\
& \leq (e^\e - 1) + e^\e \supcvstab + \eta.
\end{align*}
Since the inequality holds for any $\eta > 0$, we have
\[
\E_{\substack{\sfS \sim \sD^{pm}\\ k = \cA(\sfS)}} [\Psi(S_k, S_k)]
\leq (e^\e - 1) + e^\e \supcvstab.
\]
Thus, by \eqref{eq:ExpectationOfMaxBound}, the following inequality
holds:
\begin{equation}
\label{eq:ExpBound1}
\E_{\substack{\sfS \sim \sD^{pm}}} \Big[ \max\set[\Big]{0, \max_{k \in [p]}
  \set[\big]{\Psi(S_k, S_k)}} 
\Big] 
\leq (e^\e - 1) + e^\e \supcvstab + \frac{2 \supdiam}{\e}
\log (p + 1).
\end{equation}
For any $\delta \in (0, 1)$, choose $p = \frac{\log 2}{\delta}$, which
implies
$\log (p + 1) = \log \left[ \frac{2 + \delta}{\delta} \right] \leq
\log \frac{3}{\delta}$. Then, by Lemma~\ref{lemma:SU}, with
probability at least $1 - \delta$ over the draw of a sample
$S \sim \sD^m$, the following inequality holds for all $h \in \sH_S$:
\begin{equation}
\label{eq:tmp}
R(h) \leq \h R_S(h) + 
(e^\e - 1) + e^\e \supcvstab + \frac{2 \supdiam}{\e}
\log \left[ \frac{3}{\delta} \right].
\end{equation}
For $\e \leq \frac{1}{2}$, the inequality $(e^\e - 1) \leq 2 \e$
holds. Thus,
\[
(e^\e - 1) + e^\e \supcvstab + \frac{2 \supdiam}{\e}
\log \left[ \frac{3}{\delta} \right]
\leq 2\e + \sqrt{e} \supcvstab + \frac{2 \supdiam}{\e}
\log \left[ \frac{3}{\delta} \right]
\]
Choosing
$\e = \sqrt{\supdiam \log \left[
    \frac{3}{\delta} \right]}$ gives
\begin{align*}
R(h) 
& \leq \h R_S(h) + 
\sqrt{e} \supcvstab +
4 \sqrt{\supdiam \log \left[
    \frac{3}{\delta} \right]}\\
& = \h R_S(h) + 
\sqrt{e} \supcvstab +
4 \sqrt{\left[\tfrac{1}{m} + 2 \beta \right] \log \left[
    \tfrac{3}{\delta} \right]}.
\end{align*}
Combining this inequality with the inequality of Theorem~\ref{th:hss}
related to the Rademacher complexity:
\begin{equation}
\forall h \in \sH_S, 
R(h) \leq \h R_S(h) + 2 \R^\diamond_m(\cG) + [1 + 2 \beta m] \sqrt{\frac{\log \frac{1}{\delta}}{2m}},
\end{equation}
and using the union bound complete the proof.
\end{proof}

The following is a helper lemma for the analysis in the above proof:
\begin{lemma}
\label{lem:supcvstab-bound}
The following upper bound in terms of the CV-stability coefficient
$\supcvstab$ holds:
\[
\sum_{k = 1}^p \, \E_{\substack{\sfS \sim \sD^{pm}\\ z' \sim\sD,\ z \sim S_k}}  \bigg[e^\e \Pr[\cA(\sfS) = k] \,
  [L(h_{S_k^{z \leftrightarrow z'}}, z) -
  L(h_{S_k}, z)] \bigg] \leq e^\e \supcvstab.
\]
\end{lemma}
\begin{proof}
Upper bounding the difference of losses by a supremum to
make the CV-stability coefficient appear gives the
following chain of inequalities:
\begin{align*}
  & \sum_{k = 1}^p \, \E_{\substack{\sfS \sim \sD^{pm}\\ z' \sim\sD,\ z \sim S_k}}  \bigg[e^\e \Pr[\cA(\sfS) = k] \,
  [L(h_{S_k^{z \leftrightarrow z'}}, z) -
  L(h_{S_k}, z)] \bigg] \\
  & \leq \sum_{k = 1}^p \, \E_{\substack{\sfS \sim \sD^{pm}\\ z' \sim\sD,\ z \sim S_k}}  \bigg[e^\e \Pr[\cA(\sfS) = k] \,
  \sup_{h \in \sH_{S_k} \!, \, h' \in \sH_{S_k^{z \leftrightarrow z'}}} [L(h', z) -
  L(h, z)] \bigg]\\
  & = \sum_{k = 1}^p \, \E_{\sfS \sim \sD^{pm}}\bigg[e^\e \Pr[\cA(\sfS) = k] \,
    \E_{z' \sim\sD,\ z \sim S_k}  \Big[\sup_{h \in \sH_{S_k} \!, \, h' \in \sH_{S_k^{z \leftrightarrow z'}}} [\left. L(h', z) -
    L(h, z)]\ \right|\ \sfS \Big]\bigg] \\
  & \leq \sum_{k = 1}^p \, \E_{\sfS \sim \sD^{pm}}\bigg[e^\e
    \Pr[\cA(\sfS) = k] \, \supcvstab\bigg]\\
  & = \E_{\sfS \sim \sD^{pm}}\bigg[\sum_{k=1}^p \Pr[\cA(\sfS) = k]\bigg] \cdot e^\e \supcvstab\\
  & = e^\e \supcvstab,
\end{align*}
which completes the proof.
\end{proof}

\subsection{Proof of bound~\eqref{eq:induction-hss}}

Bound~\eqref{eq:induction-hss} is a simple consequence of the fact
that the composition of the two stages of the learning algorithm is
uniformly-stable in the classical sense. Specifically, consider a
learning algorithm that consists of determining the hypothesis set
$\sH_S$ based on the sample $S$ and then selecting an arbitrary (but
fixed) hypothesis $h_S \in \sH_S$. The following lemma shows that the
uniform-stability coefficient of this learning algorithm can be
bounded in terms of its hypothesis set stability and its max-diameter.
\begin{lemma}
  Suppose the family of data-dependent hypothesis sets
  $\cH = (\sH_S)_{S \in \sZ^m}$ is $\beta$-uniformly stable and admits
  max-diameter $\maxdiam$. Then, for any two samples $S, S' \in \sZ^m$
  differing in exactly one point, and for any $z \in \sZ$, we have
\[ |L(h_S, z) - L(h_{S'}, z)| \leq 3\beta + \maxdiam. \]  
\end{lemma}
\begin{proof}
  We first show that for any two hypotheses $h, h' \in \sH_S$ and for
  any $z \in \sZ$, we have
  $|L(h, z) - L(h', z)| \leq 2\beta + \maxdiam$. Indeed, let $S''$ be
  a sample obtained by replacing an arbitrary point in $S$ by
  $z$. Then, by $\beta$-uniform hypothesis set stability of $\cH$,
  there exist hypotheses $g, g' \in \sH_{S''}$ such that
  $|L(h, z) - L(g, z)| \leq \beta$ and
  $|L(h', z) - L(g', z)| \leq \beta$. Furthermore, since $z \in S''$,
  we have $|L(g, z) - L(g', z)| \leq \maxdiam$. By combining these
  inequalities, we get that
  $|L(h, z) - L(h', z)| \leq 2\beta + \maxdiam$, as required.

  Now, let $h' \in \sH_{S}$ be a hypothesis such that
  $|L(h', z) - L(h_{S'}, z)| \leq \beta$. Since $h', h_S \in \sH_S$,
  by the analysis in the preceding paragraph, we have
  $|L(h_S, z) - L(h', z)| \leq 2\beta + \maxdiam$. Combining these two
  inequalities, we have
  $|L(h_S, z) - L(h_{S'}, z)| \leq 3\beta + \maxdiam$, completing the
  proof.
\end{proof}

Finally, bound~\eqref{eq:induction-hss} follows immediately from the
following result of \citet{FeldmanVondrak2019}, setting
$\ell(S, z) := L(h_S, z)$ and $\gamma = 3\beta + \maxdiam$, and the
fact that any two hypotheses $h$ and $h'$ in $\sH_S$ differ in loss on
any point $z$ by at most $\maxdiam$ in order to get a bound which
holds uniformly for all $h \in \sH_S$.
\begin{theorem}[\citep{FeldmanVondrak2019}]
  Let $\ell\colon \sZ^m \times \sZ \to [0, 1]$ be a data-dependent
  function with uniform stability $\gamma$, i.e. for any
  $S, S' \in \sZ^m$ differing in one point, and any $z \in \sZ$, we
  have $|\ell(S, z) - \ell(S', z)| \leq \gamma$. Then, for any
  $\delta > 0$, with probability at least $1-\delta$ over the choice
  of the sample $S$, the following inequality holds:
\[\left|\E_{z \sim \sD}[\ell(S, z)] - \E_{z \sim S}[\ell(S, z)]\right| \leq 47\gamma \log(m)\log(\tfrac{5m^3}{\delta}) + \sqrt{\tfrac{4}{m}\log(\tfrac{4}{\delta})}.\]  
\end{theorem}

\newpage
\section{Extensions}
\label{app:extensions}

We briefly discuss here some extensions of the framework
and results presented in the previous section.

\subsection{Almost everywhere hypothesis set stability}
\label{sec:almost-everywhere}

As for standard algorithmic uniform stability, our generalization
bounds for hypothesis set stability can be extended to the case where
hypothesis set stability holds only with high probability
\citep{KutinNiyogi2002}.

\begin{definition}
  Fix $m \geq 1$. We will say that a family of data-dependent
  hypothesis sets $\cH = (\sH_S)_{S \in \sZ^m}$ is \emph{weakly
    ($\beta$, $\delta$)-stable} for some $\beta \geq 0$ and
  $\delta > 0$, if, with probability at least $1 - \delta$ over the
  draw of a sample $S \in \sZ^m$, for any sample $S'$ of size $m$
  differing from $S$ only by one point, the following holds:
\begin{equation}
\label{eq:ahss}
  \forall h \in \sH_S, \exists h' \in \sH_{S'} \colon \,
\forall z \in \sZ, | L(h, z) - L(h', z) | \leq \beta.
\end{equation}
\end{definition}

Notice that, in this definition, $\beta$ and $\delta$ depend on the
sample size $m$. In practice, we often have $\beta = O(\frac{1}{m})$
and $\delta = O(e^{-\Omega(m)})$. The learning bounds of
Theorem~\ref{th:hss} can be
straightforwardly extended to guarantees for weakly ($\beta$,
$\delta$)-stable families of data-dependent hypothesis sets, by using
a union bound and the confidence parameter $\delta$.

\subsection{Randomized algorithms}
\label{sec:randomized-algs}

The generalization bounds given in this paper assume that the
data-dependent hypothesis set $\sH_S$ is \emph{deterministic}
conditioned on $S$. However, in some applications such as bagging, it
is more natural to think of $\sH_S$ as being constructed by a
\emph{randomized} algorithm with access to an independent source of
randomness in the form of a random seed $s$. Our generalization bounds
can be extended in a straightforward manner for this setting if the
following can be shown to hold: there is a \emph{good} set of seeds,
$G$, such that (a) $\Pr[s \in G] \geq 1 - \delta$, where $\delta$ is
the confidence parameter, and (b) conditioned on any $s \in G$, the
family of data-dependent hypothesis sets $\cH = (\sH_S)_{S \in \sZ^m}$
is $\beta$-uniformly stable. In that case, for any good set $s \in G$,
Theorem~\ref{th:hss} holds. Then taking a union bound, we conclude
that with probability at least $1-2\delta$ over both the choice of the
random seed $s$ and the sample set $S$, the generalization bounds
hold. This can be further combined with almost-everywhere hypothesis
stability as in section~\ref{sec:almost-everywhere} via another union
bound if necessary.

\subsection{Data-dependent priors}
\label{sec:data-dependent-priors}

An alternative scenario extending our study is one where, in the first
stage, instead of selecting a hypothesis set $\sH_S$, the learner
decides on a probability distribution ${\mathsf p_S}$ on a fixed
family of hypotheses $\sH$. The second stage consists of using that
\emph{prior} ${\mathsf p_S}$ to choose a hypothesis $h_S \in \sH$,
either deterministically or via a randomized algorithm. Our notion of
hypothesis set stability could then be extended to that of stability
of priors and lead to new learning bounds depending on that stability
parameter. This could lead to data-dependent prior bounds somewhat
similar to the PAC-Bayesian bounds
\citep{Catoni2007,ParradoHernandezAmbroladzeShaweTaylorSun2012,LeverLavioletteShaweTaylor2013,DziugaiteRoy2018a,DziugaiteRoy2018b},
but with technically quite different guarantees.

\ignore{
Data-dependent hypothesis classes are conceptually related to the
notion of data-dependent priors in PAC-Bayesian generalization bounds
\citep{Catoni2007,ParradoHernandezAmbroladzeShaweTaylorSun2012,LeverLavioletteShaweTaylor2013,DziugaiteRoy2018a,DziugaiteRoy2018b}. This
work may be viewed as complementary to our own; in particular,
PAC-Bayes bounds apply only to randomized mixtures of hypotheses,
whereas our bounds apply to any algorithm that outputs a single
policy. 

The recent work of \cite{DziugaiteRoy2018a,DziugaiteRoy2018b} on using
data-dependent priors in PAC-Bayes bounds is motivated by similar
considerations as this paper. While both their work and ours rely on
differential privacy, technically the papers are rather different: we
use differential privacy primitives to derive tail bounds for
uniformly stable algorithms following \cite{FeldmanVondrak2018},
whereas they use differential privacy as a mechanism to choose
data-dependent priors.
}

\newpage
\section{Other applications}
\label{app:other-applications}

\subsection{Anti-distillation}

A similar setup to distillation (section~\ref{sec:distillation}) is that of \emph{anti-distillation}
where the predictor $f^*_S$ in the first stage is chosen from a
simpler family, say that of linear hypotheses, and where the
sample-dependent hypothesis set $\sH_S$ is the subset of a very rich
family $\sH$.  $\sH_S$ is defined as the set of predictors that are
close to $f^*_S$:
\[
  \sH_S = \set[\Big]{h \in \sH \colon (\| (h - f^*_S) \|_{\infty} \leq
    \gamma) \wedge (\| (h - f^*_S) 1_S \|_{\infty} \leq
    \supdiam)},
\]
with $\supdiam = O(1/\sqrt{m})$. Thus, the restriction to $S$ of a
hypothesis $h \in \sH_S$ is close to $f^*_S$ in $\ell_\infty$-norm. As
shown in section~\ref{sec:distillation}, the family of hypothesis sets $\sH_S$
is $\mu \beta$-stable.  However, here, the hypothesis sets $\sH_S$
could be very complex and the Rademacher complexity
$\R^\diamond_m(\cH)$ not very favorable. Nevertheless, by
Theorem~\ref{th:hss}, for any $\delta > 0$, with probability at
least $1 - \delta$ over the draw of a sample $S \sim \sD^m$, the
following inequality holds for any $h \in \sH_S$:
\[
R(h) \leq \h R_S(h) + 
\sqrt{e} \mu (\supdiam + \beta) +
4 \sqrt{(\tfrac{1}{m} + 2 \mu \beta) \log(\tfrac{6}{\delta})}.
\]  
Notice that a standard uniform-stability does not apply
here since the $(1/\sqrt{m})$-closeness of the hypotheses to $f^*_S$
on $S$ does not imply their global $(1/\sqrt{m})$-closeness.

\subsection{Principal Components Regression}

Principal Components Regression is a very commonly used technique in
data analysis. In this setting, $\sX \subseteq \Rset^d$ and
$\sY \subseteq \Rset$, with a loss function $\ell$ that is
$\mu$-Lipschitz in the prediction. Given a sample
$S = \set{(x_i, y_i) \in \sX \times \sY\colon i \in [m] }$, we learn a
linear regressor on the data projected on the principal
$k$-dimensional space of the data. Specifically, let
$\Pi_S \in \Rset^{d \times d}$ be the projection matrix giving the
projection of $\Rset^d$ onto the principal $k$-dimensional subspace
of the data, i.e.\ the subspace spanned by the top $k$ left singular
vectors of the design matrix $X_S = [x_1, x_2, \cdots, x_m]$. The
hypothesis space $\sH_S$ is then defined as
$\sH_S = \set{x \mapsto w^\top \Pi_Sx \colon w \in \Rset^k,\ \|w\| \leq
\gamma }$, where $\gamma$ is a predefined bound on the norm of the
weight vector for the linear regressor. Thus, this can be seen as an
instance of the setting in section~\ref{sec:feature-mappings}, where
the feature mapping $\Phi_S$ is defined as $\Phi_S(x) = \Pi_Sx$.

To prove generalization bounds for this setup, we need to show that
these feature mappings are stable. To do that, we make the following
assumptions:
\begin{enumerate}

\item For all $x \in \sX$, $\|x\| \leq r$ for some constant
  $r \geq 1$.

\item The data covariance matrix $\E_x[xx^\top]$ has a gap of
  $\lambda > 0$ between the $k$-th and $(k+1)$-th largest eigenvalues.

\end{enumerate}
The matrix concentration bound of \citet{RudelsonVershynin2007}
implies that with probability at least $1-\delta$ over the choice of
$S$, we have
$\| X_SX_S^\top - m\E_{x}[xx^\top] \| \leq cr^2 \sqrt{m \log(m)
  \log(\tfrac{2}{\delta})}$ for some constant $c > 0$. Suppose $m$ is
large enough so that
$cr^2 \sqrt{m \log(m) \log(\tfrac{2}{\delta})} \leq
\frac{\lambda}{2}m$. Then, the gap between the $k$-th
and $(k + 1)$-th largest eigenvalues of $X_SX_S^\top$ is at least
$\frac{\lambda}{2}m$. Now, consider changing one sample point
$(x, y) \in S$ to $(x, y')$ to produce the sample set $S'$. Then, we
have $X_{S'}X_{S'}^\top = X_SX_S^\top - xx^\top + x'x^{'\top}$. Since
$\| - xx^\top + x'x^{'\top} \| \leq 2r^2$, by standard matrix
perturbation theory bounds \citep{Stewart1998}, we have
$\|\Pi_S - \Pi_{S'}\| \leq O(\frac{r^2}{\lambda m})$. Thus,
$\|\Phi_S(x) - \Phi_{S'}(x)\| \leq \| \Pi_S - \Pi_{S'} \| \|x\| \leq
O(\frac{r^3}{\lambda m})$.

Now, to apply the bound of \eqref{eq:delta-sensitive-bound}, we need
to compute a suitable bound on $\R^\diamond_m(\cH)$. For this, we
apply Lemma~\ref{lemma:linear-rad}. For any $\|w\| \leq \gamma$, since
$\|\Pi_S\| = 1$, we have $\|\Pi_S w\| \leq \gamma$. So the hypothesis
set
$\sH'_S = \{x \mapsto w^\top \Pi_Sx\colon w \in \Rset^k,\ \|\Pi_Sw\|
\leq \gamma\}$ contains $\sH_S$. By Lemma~\ref{lemma:linear-rad}, we
have $\R^\diamond_m(\cH') \leq \frac{\gamma r}{\sqrt{m}}$. Thus, by
plugging the bounds obtained above in
\eqref{eq:delta-sensitive-bound}, we conclude that with probability at
least $1-2\delta$ over the choice of $S$, for any $h \in \sH_S$, we
have
\[
R(h) \leq \h R_S(h) + 
O\left(\mu \gamma \frac{r^3}{\lambda}\sqrt{\frac{\log
    \frac{1}{\delta}}{m}}\right).
\]
\newpage
\section{PAC-Bayesian Bounds}
\label{sec:PAC-Bayes}

The PAC-Bayes framework assumes a prior distribution $P$ over $\sH$
and a posterior distribution $Q$ selected after observing the training
sample. The framework helps derive learning bounds for randomized
algorithms with probability distribution $Q$, in terms of the relative
entropy of $Q$ and $P$.

In this section, we briefly discuss PAC-Bayesian learning bounds and
present some key results. In Subsection~\ref{sec:Rad->Bayes},
we give PAC-Bayes learning bounds derived from Rademacher
complexity bounds, which improve upon standard PAC-Bayes bounds
\citep{McAllester2003}.  Similar bounds were already shown by
\citet{KakadeSridharantTewari2008} using elegant proofs based on
strong convexity. Here, we give an alternative proof not invoking
strong convexity.
In Subsection~\ref{sec:data-dep-PAC-Bayes}, we extend the PAC-Bayes
framework to one where the prior distribution is selected after
observing $S$ and will denote by $P_S$ that prior.
Finally, in Subsection~\ref{sec:derandomization}, we briefly discuss
derandomized PAC-Bayesian bounds, that is learning bounds derived for
deterministic algorithms, using PAC-Bayes bounds.

\subsection{PAC-Bayes bounds derived from Rademacher complexity bounds}
\label{sec:Rad->Bayes}

We will denote by $L_z$ the vector $(L(h, z))_{h \in \sH}$. The
expected loss of the randomized classifier $Q$ can then be written as
$\E_{\substack{h \sim Q\\z \sim \sD}}[L(h, z)] = \E_{z \sim \sD} \big[
\langle Q, L_{z} \rangle \big]$.

Define $\sG_\mu$ via
$\sG_\mu = \set{Q \in \Delta(\sH) \colon \sfD(Q \parallel P) \leq
  \mu}$, that is the family of distributions $Q$ defined over $\sH$
with $\mu$-bounded relative entropy with respect to $P$. Then, by the
standard Rademacher complexity bound
\citep{KoltchinskiiPanchenko2002,MohriRostamizadehTalwalkar2012}, for
any $\delta > 0$, with probability at least $1 - \delta$ over the draw
of a sample $S$ of size $m$, the following holds for all
$Q \in \sG_\mu$:
\begin{equation}
\label{eq:23}
\E_{z \sim \sD} \big[ \langle Q, L_{z} \rangle \big]
\leq \E_{z \sim S} \big[ \langle Q, L_{z} \rangle \big]
+ 2 \R_m(\sG_\mu) + \sqrt{\frac{\log \frac{1}{\delta}}{2m}}.
\end{equation}
We now give an upper bound on $\R_m(\sG_\mu)$.  For any $Q$, define
$\Psi(Q)$ by $\Psi(Q) = \sfD(Q, P)$ if $Q \in \Delta(\sH)$ and
$+\infty$ otherwise. It is known that the conjugate function $\Psi^*$
of $\Psi$ is given by
$\Psi^*(U) = \log \big( \E_{h \in P} [e^{U(h)}] \big)$, for all
$U \in \Rset^{\sH}$ (see for example \cite[Lemma
B.37]{MohriRostamizadehTalwalkar2012}).
Let $U_\bsigma = \sum_{i = 1}^m \sigma_i L_{z_i}$. Then, for any $t >
0$, we can write:
\begin{align*}
\R_m(\sG_\mu) 
& = \frac{1}{m} \E_{S, \bsigma} \bigg[\sup_{\sfD(Q \parallel
  P) \leq \mu} \sum_{i = 1}^m \sigma_i \langle Q,
  L_{z_i} \rangle \bigg]\\
& = \frac{1}{m} \E_{S, \bsigma} \bigg[\sup_{\sfD(Q \parallel
  P) \leq \mu} \langle Q,
  U_\bsigma \rangle \bigg] & (\text{definition of $U_\bsigma$})\\
& = \frac{1}{m t} \E_{S, \bsigma} \bigg[\sup_{\sfD(Q \parallel
  P) \leq \mu} \langle Q,
  t U_\bsigma \rangle \bigg] & (t > 0)\\
& \leq \frac{1}{m t} \E_{S, \bsigma} \bigg[\sup_{\Psi(Q) \leq \mu} \Psi(Q) + \Psi^*(t
  U_\bsigma ) \bigg] & (\text{Fenchel inequality})\\
& \leq \frac{\mu}{m t}  + \frac{1}{m t} \E_{S, \bsigma} \big[\Psi^*(t
  U_\bsigma ) \big].
\end{align*}
Now, we use the expression of $\Psi^*$ to bound the second term
as follows:
\begin{align*}
\E_{S, \bsigma} \Big[\Psi^*(t U_\bsigma ) \big]
& = \E_{S, \bsigma} \big[\log \Big( \E_{h \sim P} \Big[ e^{t \sum_{i = 1}^m \sigma_i
  L(h, z_i)} \Big] \Big) \Big]\\
& \leq \E_{S} \Big[\log \Big( \E_{\bsigma, h \sim P} \Big[e^{t \sum_{i = 1}^m \sigma_i
  L(h, z_i)} \Big] \Big) \Big] & (\text{Jensen's inequality})\\
& = \E_{S} \Big[\log \Big( \E_{h \sim P} \Big[\prod_{i = 1}^m
  \cosh(t L(h, z_i)) \Big] \Big) \Big]\\
& \leq \E_{S} \Big[\log \Big( \E_{h \sim P} \Big[e^{m \frac {t^2}{2}}
  \Big] \Big) \Big] = \frac{m t^2}{2}.
\end{align*}
Choosing $t = \sqrt{\frac{2 \mu}{m}}$ to minimize the bound on the
Rademacher complexity gives
$\R_m(\sG_\mu) \leq \sqrt{\frac{2\mu}{m}}$. In view of that,
\eqref{eq:23} implies:
\begin{equation}
\label{eq:24}
\E_{z \sim \sD} \big[ \langle Q, L_{z} \rangle \big]
\leq \E_{z \sim S} \big[ \langle Q, L_{z} \rangle \big]
+ 2 \sqrt{\frac{2\mu}{m}} + \sqrt{\frac{\log \frac{1}{\delta}}{2m}}.
\end{equation}

Proceeding as in \citep{KakadeSridharantTewari2008}, by the union
bound, the result can be extended to hold for any distribution $Q$,
which is directly leading to the following result.

\begin{theorem}
  Let $P$ be a fixed prior on $\sH$. Then, for any $\delta > 0$, with
  probability at least $1 - \delta$ over the draw of a sample $S$ of
  size $m$, the following holds for any posterior distribution $Q$
  over $\sH$:
\[
\E_{\substack{h \sim Q\\z \sim \sD}}[L(h, z)]
\leq \E_{\substack{h \sim Q}}\bigg[ \frac{1}{m} \sum_{i = 1}^m L(h,
z_i) \bigg]
+ \Big( 4 + \tfrac{1}{\sqrt{e}} \Big) \sqrt{\frac{\max\set{\sfD(Q \parallel P), 1}}{m}} 
+ \sqrt{\frac{\log \frac{1}{\delta}}{2m}}.
\]
\end{theorem}
This bound improves upon standard PAC-Bayes bounds (see for example
\citep{McAllester2003}) since it does not include an additive term in
$\sqrt{(\log m)/m}$, as pointed by \cite{KakadeSridharantTewari2008}.

\subsection{Data-dependent PAC-Bayes bounds}
\label{sec:data-dep-PAC-Bayes}

In this section, we extend the framework to one where the prior
distribution is selected after observing $S$ and will denote by $P_S$
that prior. To analyze that scenario, we can both use the general
data-dependent learning bounds of Section~\ref{sec:hss2}, or the 
hypothesis set stability bounds of Section~\ref{sec:hss}. We
will focus here on the latter.

Define the data-dependent hypothesis set
$\sG_{S, \mu} = \set{Q \in \Delta(\sH) \colon \sfD(Q \parallel P_S)
  \leq \mu}$ and assume that the priors $P_S$ are chosen so that
$\cG_\mu = (\sG_{S, \mu})_S$ is $\beta$-stable. This may be by
choosing $P_S$ and $P_{S'}$ to be close in total variation or relative
entropy for any two samples $S$ and $S'$ differing by one point. Then,
by Theorem~\ref{th:hss}, for any $\delta > 0$, with probably at least
$1 - \delta$, the following holds for all $Q \in \sG_{\mu, S}$:
\begin{align*}
\E_{\substack{h \sim Q\\z \sim \sD}}[L(h, z)]
\leq \E_{\substack{h \sim Q}}\bigg[ \frac{1}{m} \sum_{i = 1}^m L(h,
z_i) \bigg] + \min\Bigg\{ & \min\left\{2 \R^\diamond_m(\cG_\mu), 
\beta + \avgdiam \right\} + (1 + 2 \beta m) \sqrt{\tfrac{1}{2m}\log(\tfrac{1}{\delta})}, \\
& \sqrt{e} (\beta + \supdiam) + 4 \sqrt{(\tfrac{1}{m} + 2 \beta)
  \log(\tfrac{6}{\delta})},\\
& 48(3\beta + \maxdiam)\log(m)\log(\tfrac{5m^3}{\delta}) + \sqrt{\tfrac{4}{m}\log(\tfrac{4}{\delta})}\Bigg\}.
\end{align*}
The analysis of the Rademacher complexity $\R^\diamond_m(\cG_\mu)$
depends on the specific properties of the family of priors
$P_S$. Here, we initiate its analysis and leave it to the reader
to complete it for a choice of the priors.

Proceeding as in Subsection~\ref{sec:Rad->Bayes}, we define $\Psi_S$
by $\Psi_S(Q) = \sfD(Q, P_S)$ for any $Q \in \Delta(\sH)$ and denote
by $\Psi_S^*$ its conjugate function.
Let $U_\bsigma = \sum_{i = 1}^m \sigma_i L_{z_i^T}$. Then, for any $t >
0$, we can write:
\begin{align*}
\R^\diamond_m(\sG_\mu) 
& = \frac{1}{m} \E_{S, T, \bsigma} \bigg[\sup_{\sfD(Q \parallel
  P_{S^\bsigma_T}) \leq \mu} \sum_{i = 1}^m \sigma_i \langle Q,
  L_{z_i} \rangle \bigg]\\
& = \frac{1}{m} \E_{S, T, \bsigma} \bigg[\sup_{\sfD(Q \parallel
  P_{S^\bsigma_T}) \leq \mu} \langle Q,
  U_\bsigma \rangle \bigg] & (\text{definition of $U_\bsigma$})\\
& = \frac{1}{m t} \E_{S, T, \bsigma} \bigg[\sup_{\sfD(Q \parallel
  P_{S^\bsigma_T}) \leq \mu} \langle Q,
  t U_\bsigma \rangle \bigg] & (t > 0)\\
& \leq \frac{1}{m t} \E_{S, T, \bsigma} \bigg[\sup_{\Psi_{S^\bsigma_T}(Q) \leq \mu} \Psi_{S^\bsigma_T}(Q) + \Psi_{S^\bsigma_T}^*(t
  U_\bsigma ) \bigg] & (\text{Fenchel inequality})\\
& \leq \frac{\mu}{m t}  + \frac{1}{m t} \E_{S, T, \bsigma} \big[\Psi_{S^\bsigma_T}^*(t
  U_\bsigma ) \big].
\end{align*}
Using the expression of the conjugate function $\Psi_{S^\bsigma_T}^*$,
as in Subsection~\ref{sec:Rad->Bayes}, the second term can be bounded
as follows:
\begin{align*}
\E_{S, T, \bsigma} \Big[\Psi_{S^\bsigma_T}^*(t U_\bsigma ) \big]
& = \E_{S, T, \bsigma} \big[\log \Big( \E_{h \sim P_{S^\bsigma_T}} \Big[ e^{t \sum_{i = 1}^m \sigma_i
  L(h, z_i)} \Big] \Big) \Big]\\
& \leq \E_{S, T} \Big[\log \Big( \E_{\bsigma, h \sim P_{S^\bsigma_T}} \Big[e^{t \sum_{i = 1}^m \sigma_i
  L(h, z_i)} \Big] \Big) \Big] & (\text{Jensen's inequality}).
\end{align*}
This last term can be bounded using Hoeffding's inequality and the
specific properties of the priors leading to an explicit bound on
the Rademacher complexity as in Subsection~\ref{sec:Rad->Bayes}.

\subsection{Derandomized PAC-Bayesian bounds}
\label{sec:derandomization}

Derandomized versions of PAC-Bayesian bounds have been given in the
past: margin bounds for linear predictors by \cite{McAllester2003},
more complex margin bounds by \cite{NeyshaburBhojanapalliSrebro2018}
where linear predictors are replaced with neural networks and where
the norm-bound is replaced with a more complex norm condition, and
chaining-based bounds by \cite{Miyaguchi2019}.

However, the benefit of these bounds is not clear since finer
Rademacher complexity bounds can be derived for deterministic
predictors. In fact, Rademacher complexity bounds can be used to
derive finer PAC-Bayes bounds than existing ones. This was already
pointed out by \cite{KakadeSridharantTewari2008} and further shown
here with an alternative proof and more favorable constants
(Subsection~\ref{sec:Rad->Bayes}).

In fact, using the technique of obtaining KL-divergence between prior
and posterior as upper bound on the Rademacher complexity, along with
the optimistic rates in \citep{SreSriTew10}, one can obtain just as in
the previous section, an optimistic rate with data-dependent prior
when one considers a non-negative smooth loss and, as predictor, the
expected model under the posterior. As this is a straightforward
application of the result of \cite{SreSriTew10} combined with
techniques presented here, we leave this for the reader to verify by
themselves.

\end{document}